\documentclass[10pt,journal,compsoc]{IEEEtran}
\usepackage[utf8]{inputenc}
\usepackage[T1]{fontenc}

\usepackage{standalone}
\usepackage[english]{babel}
\usepackage{multirow}
\usepackage{pgfplots}
\pgfplotsset{compat=1.12}
\usepgfplotslibrary{statistics}
\usepackage[tableposition=top]{caption}
\usepackage[microtype,citepackage]{pauli_new}
\usepackage{siunitx}
\usepackage{booktabs}
\usepackage{url}
\usepackage{algorithm}
\usepackage[noend]{algpseudocode}
\usepackage{pdfpages}
\usepackage{xcolor}
\usepackage{color,soul}
\renewcommand{\hl}[1]{#1}
\makeatletter
\let\MYcaption\@makecaption
\makeatother
\usepackage[font=footnotesize]{subcaption}
\makeatletter
\let\@makecaption\MYcaption
\makeatother
\newcommand*\xbar[1]{%
  \hbox{%
    \vbox{%
      \hrule height 0.5pt %
      \kern0.3ex%
      \hbox{%
        \kern-0.1em%
        \ensuremath{#1}%
        \kern-0.1em%
      }%
    }%
  }%
}

\DeclareMathOperator{\supp}{supp}
\DeclareMathOperator{\cov}{cov}
\DeclareMathOperator{\corr}{corr}

\DeclareMathOperator{\tr}{tr}

\newcommand{\NMF}[1]{\mathit{NMF}_{\text{#1}}}
\newcommand{\DNMF}[1]{\mathit{DNMF}_{\text{#1}}}

\newif\ifappendix\appendixfalse

\addto\captionsenglish{}

\allowdisplaybreaks

\begin{document}

\title{Finding Rule-Interpretable Non-Negative Data Representation}

\author{%
  Matej Mihel\v{c}i\'c,
  Pauli Miettinen%
  \IEEEcompsocitemizethanks{\IEEEcompsocthanksitem M. Mihel\v{c}i\'c is with   Department of Mathematics, University of Zagreb, Zagreb, Croatia.\protect\\
    E-mail: matmih@math.hr
    \IEEEcompsocthanksitem P. Miettinen is with University of Eastern Finland, Kuopio, Finland.\protect\\
    E-Mail: pauli.miettinen@uef.fi}%
  \thanks{Accepted for publication in IEEE Transactions on Knowledge and Data Engineering }%
}

\markboth{IEEE Transactions on Knowledge and Data Engineering}%
{Mihel\v{c}i\'c and Miettinen: Finding Rule-Interpretable Non-Negative Data Representation}

\IEEEtitleabstractindextext{%
\begin{abstract}

Non-negative Matrix Factorization (NMF) is an intensively used technique for obtaining parts-based, lower dimensional and non-negative representation. Researchers in biology, medicine, pharmacy and other fields often prefer NMF over other dimensionality reduction approaches (such as PCA) because the non-negativity of the approach naturally fits the characteristics of the domain problem and its results  are easier to analyze and understand. Despite these advantages, obtaining exact characterization and interpretation of the NMF's latent factors can still be difficult due to their numerical nature. Rule-based approaches, such as rule mining, conceptual clustering, subgroup discovery and redescription mining, are often considered more interpretable but lack lower-dimensional representation of the data. We present a version of the NMF approach that merges rule-based descriptions with advantages of part-based representation offered by the NMF. Given the numerical input data with non-negative entries and a set of rules with high entity coverage, the approach creates the lower-dimensional non-negative representation of the input data in such a way that its factors are described by the appropriate subset of the input rules. In addition to revealing important attributes for latent factors, their interaction and value ranges, this approach allows performing focused embedding potentially using multiple overlapping target labels.

 \end{abstract}

}

\maketitle

\section{Introduction}
\label{sec:introduction}

Data analysis often involves applying different techniques from statistics, data mining, machine learning and signal processing to obtain satisfactory results in the given domain. Since the data typically has high dimensionality, it is often useful to decompose it into smaller components \cite{cichocki09nonnegative}. Multiple decomposition techniques have been developed (e.g.\ SVD, PCA, NMF, ICA) with their specific properties and goals \cite{cichocki09nonnegative}. 

Non-negative matrix factorization (NMF) \cite{paatero94nmf,lee99nmf} provides parts-based representation of the non-negative input data, often found in domains such as image processing, computer vision, biology, medicine or pharmacy. Many parameters in these domains have a natural representation with non-negative values; thus  it makes sense to use the NMF approach to obtain latent factors. Non-negativity of the produced latent factors allows to partially obtain their meaning as the weighted sum of the attributes contained in the input data. Although very useful in many applications, this level of factor understanding can be inadequate in some domains (e.g. biology or medicine), where more fine-grained interpretation -- for instance, understanding a valid range of values in an attribute -- is crucial for making informed decisions and appropriate conclusions. This problem has been slightly alleviated by the introduction of sparsity constraints, which aim to obtain adequate lower-dimensional data representation by keeping weights in factor representation either very close to $1$ or very close to $0$. 

Sparsity, might not be enough, and often identifying interesting groups of observations is preferable from interpretability's point of view. 
Methods for this include clustering \cite{CoxC}, %
inductive learning \cite{Michalski83InL}, subgroup discovery \cite{Klogsen96SD}, %
redescription mining \cite{Ramakrishnan04rm}, anomaly detection \cite{Grubbs69OD} and many others. Each of these approaches has their special properties and advantages, which mostly provide deeper, interpretable understanding of some aspects of the studied domain. They identify groups of entities with some properties and many in addition offer rule-based descriptions. Although highly interpretable, these approaches are not without problems. Majority of the unsupervised rule-producing approaches construct very large rule-sets that are ultimately hard to explore. Structuring them is very beneficial when performing analyses. Predictive rule sets aim to produce smaller sets of rules, but this is not always achieved in practice. Using rules as features, although shown to provide benefits for predictive models \cite{mozina2008rectifying,vens2011random,duivesteijn2012multi}, necessarily increases the dimensionality of the original problem. 

In this work, we aim to combine the advantages of both types of data analysis. The goal is to obtain non-negative part-based representation of the input data in such a way that the resulting factors are fully interpretable and represented by a number of rules obtained by any rule-producing approach or the domain expert. %
We propose a regularized NMF approach, similar to the guided NMF \cite{GNMF}, that takes a non-negative input data and a set of rules as input and produces a lower-dimensional data representation, where each latent factor is described by a number of input rules. Rules are either represented as a matrix, where rows denote entities and columns denote rules, or the constraints derived from rules are written in the clustering matrix, where rows denote entities and columns denote clusters derived from rules. These clusters may include a mixture of different rule-based objects such as rules, subgroups or redescriptions. One fully transparent way of obtaining interesting groups may include exploratory analyses of sets containing rule-based objects (see \cite{MihelcicInterSet, MihelcicRMExploration} for the description of an interactive tool allowing exploration of redescription sets). The proposed approach provides novel insights and data analysis capabilities that were previously either hard or impossible to obtain, but are at the same time highly desirable and needed for research conducted in various scientific domains (see Section \ref{sec:usecase}). These capabilities include \emph{exact interpretation of NMF latent factors} (including exact numeric intervals or categorical values of attributes with the ability of observing and analyzing relations between attributes associated with a given factor), ability to perform \emph{guided embedding}, combining multiple sources of information (e.g. embedding genes belonging to a selected subset of organisms and having a selected subset of gene functions by following predefined criteria on phenotypes etc.), and the ability to utilize information about multiple overlapping target variables (each target label can be modelled as one rule in the set of constraints). Such targets occur in the task of multi-label classification \cite{duivesteijn2012multi,HMC}.
We provide a detailed analysis of the approach pointing out scenarios that potentially lead to some shortcomings but we also offer solutions to these problems. Furthermore, we present a real-world use-case in Section \ref{sec:usecase}.

\section{Background}
\label{sec:background}

NMF factorizes a non-negative matrix $\mX\in \nR^{m\times n}$, with $m$ rows and $n$ columns, into two non-negative matrices of smaller rank $\mF\in \nR^{m\times k}$ and $\mG\in \R_{\geq 0}^{n\times k}$, where $k\ll \min(m,n)$, so that $\mX\approx \mF\mG^{\intercal}$. The non-negativity and parts-based interpretation is important in many applications \cite{cichocki09nonnegative}.

There are many optimization approaches to compute the NMF \cite{cichocki09nonnegative}, the most popular being the multiplicative updates approach~\cite{lee99nmf}. In general, the NMF factorization is not unique. Thus, it produced unsatisfiable results in some applications. One approach to alleviate the problem, allowing a tradeoff between different requirements, is by adding regularization or penalty terms \cite{cichocki09nonnegative}. The general form of the regularized NMF optimization function is: 
\begin{equation}
\min_{\mF\geq 0,\ \mG\geq 0} \norm{\mX-\mF\mG^{\intercal}}_{F}^2 + \sum \limits_{i=0}^{l}\lambda_i \varphi_i(\mathcal{P}_i)\; .
\end{equation}
Variables $\lambda_i\geq 0,\ i=1,\dots,l$ are regularization parameters and $\varphi_i\colon \times_{h=1}^{\abs{\mathcal{P}_i}} \mathcal{D}_h\to \R$ are penalty terms enforcing predefined constraints; $\mathcal{D}_1,\ldots, \mathcal{D}_{\abs{\mathcal{P}_i}}$ are the domains of parameters contained in parameter matrices $\mathcal{P}_i$. Usually at least one penalty term constrains $\mF$ or $\mG$. 

Rule learning \cite{Furnkranz14RL} is one of the fundamental, extensively researched fields of machine learning. The overall interpretability of the output results of all methods in this field makes it ideal for data mining and knowledge discovery. We focus on two aspects of rule learning, supervised rule induction \cite{Furnkranz14RL} and descriptive rule induction (rule-based conceptual clustering) \cite{Stepp86CC,Furnkranz14RL}. The main goal in supervised rule induction is to obtain a small set of comprehensible rules that are highly predictive, given some target property. The rule-based conceptual clustering algorithms construct rules  describing related groups of examples, given some quality function (without using information about target label). Subgroup discovery \cite{wrobel1997algorithm} aims to find descriptions of groups of entities with interesting or statistically unusual distribution of a target label.
Redescription mining \cite{Ramakrishnan04rm,GMRMBook} is an unsupervised task that aims to find groups of entities that can be characterized in multiple ways and to provide re-descriptions of such entities. One can use single or multi-view data, where views are disjoint sets of attributes describing a common set of entities.  Descriptions provided by redescription mining are called redescriptions. Redescriptions are tuples of rules, where each rule contains a subset of attributes from a corresponding data view and all rules from a tuple describe a common subset or very similar subsets of entities.  Attributes within each rule of a redescription can be connected using conjunction, disjunction and negation logical operator (thus its language is more expressive than the one used in regular task of rule mining and subgroup discovery).

\section{NMF with describable latent factors}
\label{sec:nmf}

Given a rule set $\mathcal{R}$ obtained on some input dataset $\mathcal{D}$,
we use  $\supp(r)$ to denote a support set of a rule $r\in \mathcal{R}$, that is, the set of all entities described by a rule $r$. 
We also use $\cov(\mathcal{R}) = \abs*{\bigcup_{r\in \mathcal{R}} \supp(r)}/\abs{\mathcal{D}}$ to denote the fraction of entities from an input dataset described by at least one rule from $\mathcal{R}$.

To describe the NMF factors with rules, we must somehow match them. To that end, recall that clustering can be depicted as a particular kind of a matrix.

\begin{definition}
  \label{def:ideal}
  Let $C_1, \ldots, C_k$ denote a clustering of the entities of $\mathcal{D}$ (where entity might belong to multiple clusters). Matrix $\widetilde{\mF}\in\{0,1\}^{m\times k}$ is a \emph{clustering assignment matrix} if $\widetilde{\mF}_{i,j}=1$ exactly when entity $e_i\in\mathcal{D}$ belongs to cluster $C_j$. We call clustering assignment matrices \emph{ideal} matrices.
\end{definition}

As the ideal matrices are special cases of nonnegative matrices, NMF can itself be seen as a generalization of clustering. We will utilize this connection in our definition, but to get from NMF to clustering, we need to turn the non-negative values into binary assignments.

\begin{definition}
  \label{def:cluster_assignment}
  A non-negative matrix $\mF\in\nR^{m\times k}$ \emph{assigns} row $i$ (entity $e_i$) to cluster $j$ if $\mF_{i,j}/\max_{\ell}\mF_{i,\ell}\geq 0.5$. If $C_{f_1},\ldots, C_{f_k}$ is a clustering such that $C_{f_z}$ corresponds to the $z$th column of $\mF$, we say that $\mF$ \emph{induces} $C_{f_1},\ldots, C_{f_k}$ and that $C_{f_z}$ is \emph{associated} to $f_z$.
\end{definition}

Notice that our definition of clustering assignment might assign the same entity into multiple clusters. We can now define our main problem.

\begin{definition}
\label{def:DNMF}
Given a non-negative matrix $\mX\in \nR^{m\times n}$ (with $m$ entities and $n$  attributes), a set of rules $\mathcal{R}$ obtained on $\mX$, such that $\cov(\mathcal{R})\approx 1$, and a number of factors $k\in \N$, the goal of the NMF with describable factors is to find matrices $\mF\in \nR^{m\times k}$ and $\mG\in \nR^{n\times k}$ such that (1) $\mX\approx \mF \mG^{\intercal}$ and (2) if $C_{f_z}, z\in \{1,\dots, k\}$, is induced by $\mF$, then $\bigcup_{r_i\in S} \supp(r_i) \approx C_{f_z}$ for some $S\subseteq \mathcal{R}$.
\end{definition}

The crux of the above definition is that the induced clustering highlights the important entities in the factor $f_z$  and the set $S$ of rules will then provide the explanation on \emph{why} these entities are important in this factor. Alternatively, it is possible that the user already knows some important clusters and wants to explain the NMF factors using them. In that case, we can change the second condition to
\begin{quote}
  (2) Given pre-defined clusters $C^*_1, C^*_2, \ldots, C^*_k$, each $C^*_z$ must have a factor $f_z'$ such that $C^*_z \approx C_{f_z'}$. 
\end{quote}

\subsection{Algorithms}
\label{sec:algorithms}

Next, we define optimization functions that can be used to solve this task. They differ in theoretical properties, level and distribution of description quality of obtained factors.  To optimize these functions and find the desired decomposition of matrix $\mX$, we use the multiplicative updates approach \cite{lee99nmf}. We apply the rectified version \cite{cichocki09nonnegative} with $\varepsilon = 10^{-9}$ to avoid convergence problems reported in the literature \cite{Berry2007NMFP,GillisConvP}. The first optimization function \eqref{eq:dnmf1} requires predefined ideal matrix $\widetilde{\mF}$ (see Def.~\ref{def:ideal}) and a regularization parameter $\lambda \geq 0$:

\begin{equation}
\label{eq:dnmf1}
\min_{\mF\geq 0,\ \mG\geq 0} \norm*{\mX-\mF \mG^{\intercal}}_{F}^2 +\lambda \norm*{\mF - \widetilde{\mF}}_{F}^2\; .
\end{equation}
 
The second optimization function uses regularization parameter $\lambda\geq 0$, a rule constraint matrix $\mP\in \{0,1\}^{m\times \abs{\mathcal{R}}}$ and a factor cost matrix $\mA\in \nR^{k\times \abs{\mathcal{R}}}$. $\mP$ is a binary matrix whose rows represent entities and columns represent rules from $\mathcal{R}$. For $e_i\in \mathcal{D}$ and $r_j\in \mathcal{R}$,  $\mP_{i,j}=1$ iff $e_i \in \supp(r_j)$. Rows of $\mA$ represent factors and columns represent rules. Values of this matrix affect the number of non-zero elements and their magnitude in the appropriate rows of the $\mF$ matrix,  depending on the available rule-set constraints. Defining these costs as the intersection sizes between each rule and the entity cluster associated to some latent factor ($\mA_{i,j} = \abs*{C_{f_i}\cap \supp(r_j)}$) ultimately leads to obtaining $\mF$ that provides clustering as provided by the ideal matrix. This approach provides more flexibility than the one using the ideal matrix and its optimization function is defined as

\begin{equation}
\label{eq:dnmf}
\min_{\mF\geq 0,\ \mG\geq 0} \norm{\mX-\mF \mG^{\intercal}}_{F}^2 +\lambda \norm{\mA-\mF^{\intercal} \mP}_{F}^2\; .
\end{equation}

We define the multiplicative update rules for both optimization functions, although we provide thorough theoretical analysis only for the second proposed optimization function. Larger values of parameter $\lambda$ increase the importance of constraints leading to better factor descriptions, but simultaneously decreasing data representation accuracy.

Matrix $\mG$ is unaffected by the imposed constraints in both functions, thus the multiplicative update rules for $\mG$ remain unchanged:  

\begin{equation}
\mG_{j,k}\leftarrow \max\left(\varepsilon, \mG_{j,k}\cdot \frac{(\mX^{\intercal} \mF)_{j,k}}{(\mG \mF^{\intercal} \mF)_{j,k}}\right)\; .
\end{equation}

To find the multiplicative update rules for matrix $\mF$ from \eqref{eq:dnmf}, we search for the local minima of the function
\begin{equation}
  \label{eq:J}
  J(\mF) = \norm{\mX-\mF \mG^{\intercal}}_{F}^2 +\lambda \norm{\mA-\mF^{\intercal} \mP}_{F}^2\; .
\end{equation}
  This is achieved by computing $\frac{\partial J}{\partial \mF}(\mF) = 2\mF\mG^{\intercal}\mG - 2\mX\mG+2\lambda(\mP\mP^{\intercal}\mF-\mP\mA^{\intercal})$ and solving $\mF\mG^{\intercal}\mG - \mX\mG+\lambda(\mP\mP^{\intercal}\mF-\mP\mA^{\intercal}) = 0$. From this and applying a rectifier function, we get:  

\begin{equation}
 \label{eq:upF}
\mF_{i,k}\leftarrow \max\left(\varepsilon, \mF_{i,k}\cdot \frac{(\mX\mG+\lambda \cdot \mP\mA^{\intercal})_{i,k}}{(\mF\mG^{\intercal}\mG+\lambda\cdot \mP\mP^{\intercal}\mF)_{i,k}}\right)\; .
\end{equation}

The convergence proof of the proposed multiplicative updates follows the steps laid out in \cite{cichocki09nonnegative,GillisConvP} (which directly prove the convergence of multiplicative updates for $\mG$).
Following the approach in \cite{GillisConvP}, the constraints of the optimization problem (\ref{eq:dnmf}) are strengthened to $\mF\geq \varepsilon,\ \mG\geq \varepsilon$.

\begin{theorem}
The function $J$ of \eqref{eq:J} is monotonically decreasing under the update rule \eqref{eq:upF}, for any constant $\varepsilon>0,\ \mF\geq \varepsilon,\ \mG\geq \varepsilon$. Every limit point obtained using multiplicative updates (\ref{eq:upF}) is a stationary point of the strengthened optimization problem (\ref{eq:dnmf}).
\end{theorem}

\begin{proof}
By definition $\mP\geq 0$. Given positively initialized matrices $\mF$ and $\mG$ it follows that the denominator in (\ref{eq:upF}) is positive at each multiplicative update step.  To prove that $J$ is monotonically decreasing, we define an auxiliary function $Z(\mH,\ \widetilde{\mH})$ as in \cite{DingOrt06}, so that $Z(\mH,\ \widetilde{\mH})\geq J(\mH)$ and $Z(\mH,\mH)=J(\mH)$ for all $\mH,\ \widetilde{\mH}$. Let $\mH^{(t+1)} = \argmin_{\mH} Z(\mH^{(t+1)},\mH^{(t)})$, from which it follows that $L(\mH^{(t)})=Z(\mH^{(t)},\mH^{(t)})\geq Z(\mH^{(t+1)},\mH^{(t)})\geq L(\mH^{(t+1)})$. Using the inequality $\sum_{i=1}^{m} \sum_{p=1}^k ((\mA(\mS'\mB))_{i,p}\mS_{i,p}^2)/(\mS_{i,p}')>\tr(\mS^{\intercal}\mA\mS\mB)$ proved in \cite{DingOrt06} for symmetric matrices $\mA\in \R^{n\times n}_{\geq 0}$ and $\mB\in \R^{k\times k}_{\geq 0}$ and matrices $\mS\in \R^{n\times k}_{\geq 0}$ and $\mS'\in \R^{n\times k}_{\geq 0}$, we show that $Z(\mF,\ \mF') = -\sum_{i,k}2(\mF^{\intercal}\mX\mG)_{i,k}-\sum_{i,k}2\lambda(\mA^{\intercal}\mF^{\intercal}\mP)_{i,k}+\lambda\sum_{i,k}(\mP^{\intercal}\mF\mF^{\intercal}\mP)+\sum_{i,k}\frac{\mF'(\mG^{\intercal}\mG)_{i,k}\mF_{i,k}^2}{\mF_{i,k}'}$ is an auxiliary function for $J$. Given $J'(\mF) = \tr(\lambda \mP^{\intercal}\mF\mF^{\intercal}\mP-2\lambda\mA^{\intercal}\mF^{\intercal}\mP-2\mF^{\intercal}\mX\mG+\mG^{\intercal}\mG\mF^{\intercal}\mF)$, it can easily be seen that $Z(\mF,\mF')=L'(\mF)$ for $\mF'=\mF$ and $Z(\mF,\mF')>L'(\mF)$, otherwise. Thus, the function $J$ is monotonically decreasing.

Next, we show that every limit point obtained using the proposed multiplicative updates is a stationary point of the strengthened optimization problem. Following the proof from \cite{GillisConvP}, given a limit point $(\xbar{\mF},\xbar{\mG})$ of a sequence $\{(\mF^k,\mG^k)\}$ and using the fact that $J$ is bounded from below, due to monotonicity, function $J$ converges to $\norm*{\mX-\bar{\mF} \bar{\mG}^{\intercal}}_{F}^2 +\lambda\norm*{\mA-\overline{\mF}^{\intercal} \mP}_{F}^2$. Also, $\overline{\mF}_{i,k} = \max(\varepsilon,\alpha_{i,k}\overline{\mF}_{i,k})$ for all $i,k$ where $\alpha_{i,k} = \overline{\mF}_{i,k} \frac{(\mX\overline{\mG}+\lambda \mP\mA^{\intercal})_{i,k}}{(\bar{\mF}\bar{\mG^{\intercal}}\bar{\mG}+\lambda \mP\mP^{\intercal}\overline{\mF})_{i,k}}$. Now $\bar{\mF}\bar{\mG^{\intercal}}\bar{\mG}>0$, so the update is well defined. By observing the stationarity conditions (KKT optimality conditions), it must be that $\overline{\mF}_{i,k}\geq \varepsilon$, $\alpha_{i,k}\leq 1$ and $(\overline{\mF}_{i,k}-\varepsilon)\cdot (\alpha_{i,k}-1) = 0$. The proposed updates preserve these properties.
\end{proof}

The derivation of the update rules and the corresponding convergence proof for $\mF$ from \eqref{eq:dnmf1}  follow the same logic as for $\mF$ from \eqref{eq:dnmf} thus they are omitted. Update rules for $\mF$ from \eqref{eq:dnmf1} are:

\begin{equation}
 \label{eq:upFO2}
\mF_{i,k}\leftarrow \max\left(\varepsilon,\ \mF_{i,k}\cdot \frac{(\mX\mG+\lambda \cdot \widetilde{\mF})_{i,k}}{(\mF\mG^{\intercal}\mG+\lambda\cdot \mF)_{i,k}}\right) \; .
\end{equation}

As we show in detail in the Appendix (Section A), the matrix $\mF$ from equation \ref{eq:dnmf} is susceptible to numerical compensations. The solution is to normalize $\mF$, which unfortunately introduces heavy optimization constraints.

The proposed solution is to combine methodologies presented in equations \ref{eq:dnmf1} and \ref{eq:dnmf}, which solves the problem of numerical compensations with manageable optimization constraints.
\begin{equation}
\label{eq:dnmf2}
\min_{\mF\geq 0,\ \mG\geq 0} \norm{\mX-\mF \mG^{\intercal}}_{F}^2 +\lambda \norm*{\mA - \mF^{\intercal}\mP+( \widetilde{\mF}-\mF)^{\intercal}\mP}_{F}^2\; .
\end{equation}
Or equivalently: 
\begin{equation}
\label{eq:dnmf3}
\min_{\mF\geq 0,\ \mG\geq 0} \norm{\mX-\mF \mG^{\intercal}}_{F}^2 +\lambda \norm*{\mA + ( \widetilde{\mF}-2\mF)^{\intercal}\mP}_{F}^2\; .
\end{equation}

To find the multiplicative update rules for matrix $\mF$ from \eqref{eq:dnmf2} and \eqref{eq:dnmf3}, we search for the local minima of the function:
\begin{equation}
  \label{eq:J1}
  J(\mF) = \norm{\mX-\mF \mG^{\intercal}}_{F}^2 +\lambda \norm*{\mA - \mF^{\intercal}\mP+( \widetilde{\mF}-\mF)^{\intercal}\mP}_{F}^2\;  .
\end{equation}
  This is achieved by computing $\frac{\partial J}{\partial \mF}(\mF) = -2\mX\mG + 2\mF\mG^{\intercal}\mG + \lambda(8\mP\mP^{\intercal}\mF - 4\mP\mA^{\intercal}-4\mP\mP^{\intercal}\widetilde{\mF})$ and solving $-2\mX\mG + 2\mF\mG^{\intercal}\mG + \lambda(8\mP\mP^{\intercal}\mF - 4\mP\mA^{\intercal}-4\mP\mP^{\intercal}\widetilde{\mF}) = 0$. From this and applying a rectifier function, we get:  

\begin{equation}
 \label{eq:upF3}
\mF_{i,k}\leftarrow \max\left(\varepsilon, \mF_{i,k}\cdot \frac{(\mX\mG+2\lambda \cdot (\mP\mA^{\intercal}+\mP\mP^{\intercal}\widetilde{\mF}))_{i,k}}{(\mF\mG^{\intercal}\mG+4\lambda\cdot \mP\mP^{\intercal}\mF)_{i,k}}\right)\; .
\end{equation}

The main benefit of this optimization function is: 
\begin{theorem} If $\mA\neq 0$ then
$\mA - \mF^{\intercal}\mP+( \widetilde{\mF}-\mF)^{\intercal}\mP = 0 \Leftrightarrow \widetilde{\mF}-\mF = 0$.
\end{theorem}

\begin{proof}
Assuming $\widetilde{\mF}-\mF = 0$, it follows: 
$\mA - \mF^{\intercal}\mP+( \widetilde{\mF}-\mF)^{\intercal}\mP = \mA - \mF^{\intercal}\mP$.
$\mP$ describes redescription supports and $\mF^{\intercal} = \widetilde{\mF}^{\intercal}$, thus $\mA - \mF^{\intercal}\mP = \mA - \widetilde{\mF}^{\intercal}\mP$. By definition, this equals $\mA - \mA = 0$.
\par \vspace{2mm}\noindent
Conversely, assume $\mA - \mF^{\intercal}\mP+( \widetilde{\mF}-\mF)^{\intercal}\mP = 0$. Suppose $(\widetilde{\mF}-\mF)^{\intercal}\neq 0$, more concretely $(\widetilde{\mF}-\mF)^{\intercal} = \mK$,  where $F_{i,j}\geq 0,\ K_{i,j}\in \mathbb{R}$. It must be: $\mA - \mF^{\intercal}\mP +\mK\mP = 0$. $\mA - \mF^{\intercal}\mP = -\mK\mP$. $\mA = -\mK\mP + \mF^{\intercal}\mP$. That is, $\mA = (-\mK+\mF^{\intercal})\mP$. Since $\mK + \mF^{\intercal} = \widetilde{\mF}^{\intercal}$, it must be: 
$\widetilde{\mF}_{i,j} = 1 \Rightarrow \mK_{i,j}+\mF_{i,j} = 1 \Rightarrow \mF_{i,j} = 1- \mK_{i,j}$. Since $\mF_{i,j}\geq 0 \Rightarrow \mK_{i,j}\leq 1$. $\widetilde{\mF}_{i,j} = 0 \Rightarrow \mF_{i,j} = - \mK_{i,j}$. Since $\mF_{i,j}\geq 0 \Rightarrow \mK_{i,j}\leq 0$.

Now $\mA = (-\mK+\widetilde{\mF}^{\intercal}-\mK)\mP = (-2\mK + \widetilde{\mF}^{\intercal})\mP$. Given observations stated above it follows: 
$-2\mK_{i,j}+\widetilde{\mF}_{i,j}<0$ if $\widetilde{\mF}_{i,j} = 1$ or $-2\mK_{i,j}+\widetilde{\mF}_{i,j}\leq 0$ if $\widetilde{\mF}_{i,j} = 0$. Given $\mP_{i,j}\geq 0$, it is easy to see that $((-2\mK + \widetilde{\mF}^{\intercal})\mP)_{i,j}\leq 0, \ \forall i,j$. This is a contradiction and thus $(\widetilde{\mF}-\mF)^{\intercal} = 0$.
\end{proof}

The convergence proof of update rules for $\mF$ from (\ref{eq:dnmf2}) follows the outline of the proof for $\mF$ from (\ref{eq:dnmf}), thus it is omitted. Approaches from equations \ref{eq:dnmf1} and \ref{eq:dnmf}  have similar optimization function to the semi-supervised NMF \cite{SSNMF}, that provides label information in the penalty term, and the concurrently developed approach of guided NMF \cite{GNMF,GuidedNMFF}, that provides information about the distribution of words (data attributes) in the predefined (selected) topics. Penalty terms in the proposed approaches represent rule-based constraints that describe NMF latent factors. Thus, matrices occurring in the proposed optimization functions have very different meaning and dimensions compared to semi-supervised and guided NMF approaches. Section \ref{sec:usecase} demonstrates using the proposed approach and data fusion to improve embedding quality. 

We presented approaches utilizing update rules for $\mF$ from \eqref{eq:upF}, further denoted $\DNMF{ind}$, that utilizes explicit information about the entity membership in support sets of rules and the input entity-factor cost matrix, from eq. \ref{eq:upFO2}, denoted $\DNMF{dir}$, which utilizes information about the final entity-factor assignment, and the combined approach denoted $\DNMF{comb}$ utilizing rules described in \eqref{eq:upF3}. Table \ref{tab:updrules} contains update rules for $\mF$ using gradient descent (GD) based approaches ($\DNMF{indG}$, $\DNMF{dirG},\ \DNMF{combG}$), GD with bold driver heuristic ($\DNMF{indGBD}$, $\DNMF{dirGBD},\ \DNMF{combGBD}$), the oblique ($\DNMF{indOB},\ \DNMF{dirOB}$, $\DNMF{combOB}$) and the HALS approaches ($\DNMF{indH},\ \DNMF{dirH},\ \DNMF{combH}$).  Optimization approaches for NMF are described in \cite{cichocki09nonnegative} and empirical comparison of several approaches in \cite{NMFSurv}. 

\begin{table}[tpb]
	\centering
	\caption{Update rules for different NMF approaches with rule-describable factors.}
	\label{tab:updrules}
	\begin{tabular}{@{}
			l%
			c%
			@{}}
		\toprule
		Approach & Update rules   \\
		\midrule
		(1) $\DNMF{indG}$ & $\mF = [(\mF - (\mF\mG^{\intercal}\mG - \mX\mG + (\mP\mP^{\intercal}\mF-$\\
		&  $\mP\mA^{\intercal}) \cdot \lambda) \cdot \gamma_F)]_{+}$ \\
		(2) $\DNMF{dirG}$ & $\mF = [(\mF - (\mF\mG^{\intercal}\mG - \mX\mG +(\mF-\widetilde{\mF})\cdot \lambda)$ \\
		& $\cdot \gamma_F)]_{+}$ \\
		(3) $\DNMF{combG}$ & $\mF = [(\mF - (\mF\mG^{\intercal}\mG - \mX\mG +(2\mP\mP^{\intercal}\mF-$ \\
		& $\mP\mA^{\intercal} -\mP\mP^{\intercal}\widetilde{\mF} )\cdot \lambda) \cdot \gamma_F)]_{+}$ \\
		(4) $\DNMF{indGBD}$ & as (1) with $\gamma_F\leftarrow \gamma_F\cdot \{\rho,\ \sigma\}$ \\ 
		(5) $\DNMF{dirGBD}$  & as (2) with $\gamma_F\leftarrow \gamma_F\cdot \{\rho,\ \sigma\}$ \\
		(6) $\DNMF{combGBD}$  & as (3) with $\gamma_F\leftarrow \gamma_F\cdot \{\rho,\ \sigma\}$ \\
		(7) $\DNMF{indOB}$ & $\mF_{k+1}^{\intercal} = [(\mF_{k}^{\intercal}-\eta_{k}\cdot(\mG^{\intercal}\mG\mF^{\intercal}-\mG^{\intercal}\mX^{\intercal}+$ \\
		&  $\lambda(\mF^{\intercal}\mP\mP^{\intercal}-\mA\mP^{\intercal})))]_{+}$ \\
		(8) $\DNMF{dirOB}$ & $\mF_{k+1}^{\intercal} = [(\mF_k^{\intercal} - \eta_{k}\cdot (\mG^{\intercal}\mG\mF^{\intercal}-\mG^{\intercal}\mX^{\intercal}+$ \\
		& $\lambda(\mF^{\intercal} - \widetilde{\mF}^{\intercal})))]_{+}$ \\
		(9) $\DNMF{combOB}$ & $\mF_{k+1}^{\intercal} = [(\mF_{k}^{\intercal}-\eta_{k}\cdot(\mG^{\intercal}\mG\mF^{\intercal}-\mG^{\intercal}\mX^{\intercal}+$ \\
		&  $\lambda(2\mF^{\intercal}\mP\mP^{\intercal}-\mA\mP^{\intercal}-\widetilde{\mF}^{\intercal}\mP\mP^{\intercal})))]_{+}$ \\
		(10) $\DNMF{indH}$ & $f_j = \frac{1}{f_j^{\intercal}f_j}[X^{(j)}g_j+\lambda(\mP a_j^{\intercal}-\mP\mP^{\intercal}f_j)]_{+}$ \\
		(11) $\DNMF{dirH}$ & $f_j =\frac{1}{f_j^{\intercal}f_j}[X^{(j)}g_j+\lambda(\widetilde{f}_j - f_j)]_{+}$\\
		(12) $\DNMF{combH}$ & $f_j = \frac{1}{f_j^{\intercal}f_j}[X^{(j)}g_j+\lambda(\mP a_j^{\intercal} + \mP\mP^{\intercal}\widetilde{f_j} $ \\
		& $-2\mP\mP^{\intercal}f_j)]_{+}$ \\
		\bottomrule
	\end{tabular} 
\end{table}

\section{Novel algorithm for rule clustering and assignment to NMF factors}
\label{sec:clust}

Rule-assignment to clusters can be done e.g.\ via $k$-means. When unsupervised rules are used or specific rules are selected as the subgroups, it is sufficient to use $k=\mathit{numFactors}$ and cluster rules based on the similarity of the entities they describe. When supervised rules are used, one needs to decompose $\mathit{numFactors}$ to $c$ groups (where $c$ equals number of different target labels). This step requires experimentation when $k$-means algorithm is used. 

In this section, we describe a clustering approach that alleviates  the need for factor decomposition and separate application of a clustering algorithm regardless of rule type. The main idea is to first assign $c$ factors with the largest entity Jaccard index to the most accurate rule (representative) of each class label, then assign the $\mathit{numFactors} - c$ factors to the rules, among remaining rules, that are the most dissimilar to the chosen representative rules of the first assigned $c$ factors  and have the largest entity Jaccard index with a chosen factor and then assign the remaining rules to factors based on the closeness of a given rule to the assigned centroid (representative rule) of a factor. Closeness is defined as a entity Jaccard index between the centroid rule and a rule for which assignment is required. Detailed description of this approach is provided in Algorithm \ref{alg:costumClustering}.

\begin{algorithm}[ht!]
\caption{Algorithm for rule assignment to factors}\label{alg:costumClustering}
\begin{algorithmic}[1]
\Require{Initial factors $F_{\mathrm{init}}$, number of factors $k$, rule set $\mathcal{R}$, number of different target class labels $c$}
\Ensure{Rule-factor assignment $\mathcal{C}$, final factors $F_{\mathrm{fin}}$}
\Procedure{RFA}{}
\If{($\mathit{ruleType}=\text{supervised}$)}
\For{($i=0;i< c;i{++}$)}
\State $R_{\mathrm{best},c}\leftarrow \operatorname{findBestRule}(c)$
\State $\mathcal{C}_s\leftarrow \mathcal{C}_s\cup \operatorname{assignRule}(R_{\mathrm{best},c},F_{\mathrm{init}})$, $s<k$
\EndFor
\Else
\State $R_{\mathrm{best}}\leftarrow \operatorname{findBestRule}()$
\State $\mathcal{C}_s\leftarrow  \mathcal{C}_s \cup  \operatorname{assignRule}(R_{\mathrm{best}},F_{\mathrm{init}})$,  $s<k$
\EndIf
\While{$\exists \operatorname{unassigned}(F_{\mathrm{init}})$}
\State $R_{\mathrm{far}}\leftarrow \min_{R\in \mathcal{R}\setminus \operatorname{centroids}(\mathcal{C})}($ \par
\hfill$\max_{R_i\in \operatorname{centroids}(\mathcal{C})}(\mathit{JS}(R,R_i)))$
\State $\mathcal{C}_s\leftarrow  \mathcal{C}_s \cup  \operatorname{assignRule}(R_{\mathrm{far}}, F_{\mathrm{init}}\setminus \operatorname{assigned}(F_{\mathrm{init}}))$,\par
\hfill $s\notin \operatorname{index}(\operatorname{assigned}(F_{\mathrm{init}}))$
\EndWhile
\For{$R\in \mathcal{R}\setminus \operatorname{assigned}(R)$}
\State $f_{\mathrm{init},i}\leftarrow \operatorname{factor}(\argmax_i JS(R, \operatorname{centroids}(\mathcal{C})_i))$
\State  $\mathcal{C}_i\leftarrow  \mathcal{C}_i \cup \operatorname{assignRule}(R,f_{\mathrm{init},i})$
\EndFor
\State $F_{\mathrm{fin}}\leftarrow \operatorname{computeFinalFactorSupports}(\mathcal{C})$
\State \textbf{return} $\mathcal{C}, F_{\mathrm{fin}}$
\EndProcedure
\end{algorithmic}
\end{algorithm} 

The initial factors $F_{\mathrm{init}}$ are derived from the randomly generated matrix $F$. These are used as a template to assign rules and changed later based on support sets of the assigned rules (alternative approach would be to initially assign selected rules to factors at random).
Function $\operatorname{findBestRule}()$ locates the most accurate rule given a rule-type metric (accuracy for supervised rules, SSE for descriptive rules, $\chi^2$ for subgroups and Jaccard index for redescriptions). Function $\operatorname{assignRule}(R, F)$ finds a factor $f\in F$ for a rule $R$ such that $\abs{\supp(R)\cap \supp(f)}/\abs{\supp(R)}$ is maximised. That is, entities described by a rule are mostly members of a factor $f$. Function $\operatorname{assignRule}(R, f_i)$ assigns a rule $R$ to factor $f_i$. This factors must contain a centroid that is the most similar with respect to a support set to a given rule $R$, that is $\mathit{JS}(R,\operatorname{centroids}(\mathcal{C})_i)$ is maximal. Function $\operatorname{computeFinalFactorSupports}()$ assigns entities to factors based on the assigned rules.

\section{Experiments}
\label{sec:experiments}

\hl{The proposed approaches are evaluated on fifteen different datasets. The main paper contains results for eight datasets, described in Table}~\ref{tab:datasets}. \hl{Problem choice is elaborated in Section D.1 and a full set of results in Section F of the Appendix. All references to alphabetical section numbers will refer to the appendix.}

\begin{table}[tpb]
  \centering
  \caption{Dataset details. {--} denotes unsupervised data.}
  \label{tab:datasets}
\begin{tabular}{@{}
  l%
  l%
  S[table-format=5.0,group-minimum-digits=4]%
  S[table-format=5.0,group-minimum-digits=4]%
  S[table-format=1.0]%
  r%
  @{}}
  \toprule
  name & code & {\#row} & {\#col} & {\#class} & ref.  \\
  \midrule
  Abalone & AB & 4177 & 7 &  2 & \cite{KaggleAbalone}\\
  Arrhythmia & AR & 452  & 106  & 2 &  \cite{Dua2019UCIML} \\
  Brain cancer & BRC & 130 & 54676 &  5 & \cite{KaggleBrainCancer} \\
  Nomao & NM & 34465 & 89 & 2 &  \cite{NOMAO}\\
  Sports articles & SA & 1000 & 57 &  2 &   \cite{hajj2019subjectivity}\\
  $4$News & $4$N & 400 & 800 &  {--} &\cite{MiettinenPhD}\\
  Gene functions & BO & 3475 & 4523 &  {--} &\cite{MihelcicBio}\\
 World Countries& WC & 141 & 328 &  {--} &\cite{Trade}\\
  \bottomrule
\end{tabular} \end{table}

The goal of our experiments is to assess the trade-off between representation accuracy, descriptive accuracy, correspondence and sparsity achieved by the proposed approaches. We also present examples (Section~\ref{sec:exp:example}) to show how our model improves the interpretability of the results.

\subsection{Experimental setup}
Four types of rules were used in the experiments: classification rules, descriptive rules (as in rule-based conceptual clustering), subgroups and redescriptions (BO, PH and WC datasets contain two views, thus these are used only with redescriptions). Details on rule creation and data preprocessing can be found in \hl{Section C}.   

Initial $\mF$ and $\mG$ are random matrices with values in the $[0,1]$ interval, sampled from the uniform distribution. All occurrences of $0$ in these matrices were replaced with $0.1$. Ten different randomizations were used for each experiment. Obtained rules, after filtering, were used to create the rule-representing matrix $\mP$. Supervised rules were grouped by the target class and then each group was clustered using the $k$-means algorithm, with a maximum of $100$ iterations, into a predefined, potentially mutually different number of groups.  Descriptive rules, subgroups and redescriptions were grouped using the $k$-means algorithm into a predefined number of groups. \hl{Choice of the parameter $k$ is explained in Section~D.2}. Entities described by rules contained in the obtained groups constitute the entity cluster associated with the corresponding NMF latent factor.

Elements of $\mA$ were defined as $\mA_{i,j} = \abs*{C_i\cap \supp(r_j)}$.
For these experiments, we use the second version of Def.~\ref{def:DNMF}.
The correspondence  $\corr(f_i)$ between data representation and a set of provided rules for a factor $f_i$ is obtained by computing the Jaccard coefficient between  $C_{f_i}$ and $C_{f_i}^*$, the cluster associated to $f_i$ of the corresponding ideal matrix $\widetilde{\mF}$. NMF algorithms are evaluated using the \emph{data representation error} $100\cdot\norm*{\mX-\mF \mG^{\intercal}}_{F}/\norm*{\mX}_F$ and the \emph{constraint representation error} (also called the \emph{description error}) $100\cdot\norm*{\mF_c-\widetilde{\mF}}_{F}/\norm*{\widetilde{\mF}}_F$, where $\mF_c$ is the clustering assignment matrix obtained by the NMF approach. A pair of NMF algorithms ($\mathit{nmf}_1$, $\mathit{nmf}_2$) is compared using  the average pairwise difference in correspondence (ADC) between data representation and the provided rules:   $\sum_{i=1}^{k} \bigl(\corr(f_{i, \mathit{nmf}_1}) - \corr(f_{i, \mathit{nmf}_2})\bigr)/k \in [-1,1]$. We use $\mathit{nmf}_1 = \DNMF{ind}$ and $\mathit{nmf}_2$ is the algorithm to be evaluated. The rational behind these measures is explained in \hl{Section~E}.
We evaluated the performance of three proposed optimization functions ($\DNMF{dir}$, $\DNMF{ind}$, $\DNMF{comb}$).  We tested the performance of these functions when used in combination with the proposed algorithm for rule assignment to factors (Algorithm \ref{alg:costumClustering}). These approaches are denoted $\DNMF{dirA}$, $\DNMF{indA}$, $\DNMF{combA}$.

\hl{The proposed approaches were compared with the regular NMF (multiplicative updates) ($\NMF{MU}$), to get the estimate of data representation error without using the proposed regularizer, the NMF (multiplicative updates) guaranteeing a predefined sparseness level ($\NMF{sp}$) on $\mF$} \cite{Hoyer04Spars}\hl{, to assess the impact of sparseness constraints on data representation (since the proposed regularizer enforces sparseness as well as additional interpretability-related constraints), and the approach proposed by Slawski et al.} \cite{Slawski13NMFBC} \hl{($\NMF{bf}$) that factorizes initial matrix so that $\mG$ contains Boolean components, which are more interpretable than numeric components, to compare the quality of representation with the most related methodology from the interpretability viewpoint.} Sparseness constraint imposed on $\mF$, for the $\NMF{sp}$, equals the average row sparseness of the ideal matrix $\widetilde{\mF}$ for a given problem.

\sisetup{detect-weight=true,detect-inline-weight=math}

\begin{table}[tb]
	\centering
	\caption{Comparative results of different NMF approaches, using $10$ different random initializations, when supervised rules are used as constraints. $k$ contains the numbers of latent factors per class separated by semicolon; \#iters is the average number of iterations until convergence; $RE$ is  the average representation error; $DE$ is the average description error; $ADC$ is the average pairwise difference in correspondence.}
	\label{tab:res}
\begin{tabular}{@{}
  l%
  L%
  S[table-format=5.1,group-minimum-digits=4]%
  S[table-format=2.2]%
  S[table-format=3.2]%
  S[table-format=+1.2]%
  L%
  @{}}
  \toprule
  $\mathcal{D}$ & k & {\#iters} & {$\mathit{RE}$}  &  {$\mathit{DE}$} & {$\mathit{ADC}$} & \text{algorithm}  \\
  \midrule
AB & 2;3 & 11201.0 & 2.53 & 68.54 & 0.00 & \DNMF{ind}  \\
 & & 18174.0 & 2.63 & 71.93 & -0.02 & \DNMF{indA}\\
  & & 3108.1 & 3.00 & 79.49 &  0.00 & \DNMF{dir} \\
  & & 5281.0 & 2.75 & 88.70 & 0.00 & \DNMF{dirA}\\
  & & 9855.0 & 2.79 & \bfseries 66.77 & \bfseries -0.03 & \DNMF{comb}\\
  & &  14287.0 & \bfseries 1.93 & 152.08 &  0.22 & \NMF{MU} \\
  & & 123.7 & 94.14 & 144.56 &  0.14 & \NMF{sp} \\
  & & 834.0 & 21.62 & 161.72 &  0.28 & \NMF{bf} \\
  [.7ex] %
   AR & 7;13 & 2806.0 & 18.74 & 71.27 & 0.00 & \DNMF{ind} \\
    & & 4868.0 & 19.29 & 69.84 & -0.01 & \DNMF{indA}\\
 & &  2286.0 & 17.60 & \bfseries 29.13 &  \bfseries -0.30 & \DNMF{dir} \\
 & & 16788.1 & 17.60 & 30.13 & -0.27 & \DNMF{dirA}\\
 & & 4089.0 & 19.28 & 70.33 & -0.02 & \DNMF{comb}\\
  & & 28318.8 & \bfseries 11.18 & 133.62 & 0.43 & \NMF{MU}\\
 & & 47358.9 & 15.68 & 110.70 & 0.46 & \NMF{sp}\\
 & & 21.0 & 20.45 & 161.60 & 0.35 & \NMF{bf}\\[.7ex]  %
   BRC & 1;2;1;2;1 & 9999.0 & 6.49  & 69.12 & 0.00  & \DNMF{ind} \\
    & & 9999.0 & 6.49 & 65.45 & -0.05 & \DNMF{indA}\\
  & & 9999.0 & 6.48  & 82.68  &  0.05  & \DNMF{dir} \\
  & & 9999.0 & 6.48 & 75.40 & -0.03 & \DNMF{dirA}\\
  & & 9999.0 & 6.51 & \bfseries 60.99 & \bfseries -0.06 & \DNMF{comb}\\
 & & 7953.8 & \bfseries 6.48  & 113.11  & 0.27  & \NMF{MU} \\
 & & 10000.0 & 7.21 & 104.98 & 0.38& \NMF{sp}\\
  & & 17.0 & 7.73 & 97.48 & 0.10 & \NMF{bf} \\[.7ex]  %
   NM & 23;17 & 9999.0 & 8.32  & 66.73 & 0.00  & \DNMF{ind} \\
    & & 9580.4 & 5.16 & 68.39 & 0.15 & \DNMF{indA}\\
  & & 7382.4 & 8.00  & \bfseries 17.18  &  \bfseries -0.32  & \DNMF{dir} \\
  & & 9804.4 & 5.27 & 19.92 & -0.12 & \DNMF{dirA}\\
  & & 9999.0 & 9.46 & 66.98 & -0.02 & \DNMF{comb}\\
 & & 10000.0 & \bfseries 2.53  & 146.56  & 0.38  & \NMF{MU} \\
 & & 10000.0 & 7.01 & 110.13 & 0.43& \NMF{sp}\\
 & & 500.0 & 48.31 & 167.40 & 0.35& \NMF{bf}\\[.7ex]
  %
  SA & 9;11 & 18761.0 & 4.77 & 73.35 & 0.00 & \DNMF{ind} \\
   & & 10425.0 & 4.56 & 68.41 & -0.02 & \DNMF{indA}\\
  & & 26419.8 & 4.48 & 18.98 &  \bfseries -0.40 & \DNMF{dir} \\
  & & 43839.4 & 3.85 & \bfseries 18.50 & -0.30 & \DNMF{dirA}\\
  & & 16646.0 & 5.12 & 71.90 & -0.01 & \DNMF{comb}\\
  & & 47317.8 & \bfseries 1.69 & 118.95 & 0.32 & \NMF{MU} \\
  & & 50000.0 & 2.61 & 110.06 &0.37 & \NMF{sp}  \\
  & & 49.0 & 29.24 & 110.23 & 0.26 & \NMF{bf}  \\[.7ex]  %
  %
  \bottomrule
 \end{tabular}  \end{table} 

The NMF algorithms were run for \num{50000} iterations, except on the Nomao and the Brain Cancer dataset, where they were run for \num{10000} iterations, \hl{see Section~D.3 for more details}. The tolerance was set to $10^{-8}$. The number of latent factors used for each dataset (split by the target label value in case of supervised rules) and the values of the evaluation measures for all performed experiments are provided in Tables \ref{tab:res}, \ref{tab:res1}, \ref{tab:res2} and \ref{tab:res3}. All approaches except $\NMF{bf}$ (which computes the initial matrices internally) were computed using the same initial randomly generated matrices $\mF$ and $\mG$. Regularization parameter $\lambda$, used in the experiments, for the $\DNMF{ind}$ and $\DNMF{comb}$ approach is of the form $c\norm*{\mX}_F/\norm*{\mA}_F$ and for the $\DNMF{dir}$ of a form $c\norm{\mX}_F$, where $c\in \R_{\geq 0}$ \hl{(see Section~D.4)}. The exact values of $c$ and the average row sparseness (as defined in \cite{Hoyer04Spars}) of ideal matrices for the datasets \hl{are discussed in Section~C}. Experimental results were first obtained for the $\NMF{MU}$ and the $\NMF{sp}$, and the regularizer for the $\DNMF{ind}$ was chosen to obtain as accurate factor description as possible with a decent representation error ($<0.15\cdot\norm{\mX}_F$ or $<3\cdot RE_{\NMF{MU}}$). Regularizers for the $\DNMF{dir}$ and $\DNMF{comb}$  were chosen to dominate the performance results of $\DNMF{ind}$ or to get similar descriptive performance with as small as possible decrease in representation accuracy.

\subsection{Results}

\begin{table}[tb]
	\centering
	\caption{Comparative results of different NMF approaches, using $10$ different random initializations and descriptive rules as constraints. Column names are as in Table \ref{tab:res}.}
	\label{tab:res1}
\begin{tabular}{@{}
  l%
  S[table-format=2.0]%
  S[table-format=5.1,group-minimum-digits=4]%
  S[table-format=2.2]%
  S[table-format=3.2]%
  S[table-format=+1.2]%
  L%
  @{}}
  \toprule
  {$\mathcal{D}$} & {$k$} & {\#iters} & {$\mathit{RE}$}  &  {$\mathit{DE}$} & {$\mathit{ADC}$} & \text{algorithm}  \\
  \midrule
AB & 5 & 3357.4 &3.70  & 54.00 & 0.00 & \DNMF{ind} \\
 & & 1746.2 & 3.66 & 43.46 & -0.11 & \DNMF{indA}\\
& & 4419.0 & 3.67 & \bfseries 36.08 &  -0.14 & \DNMF{dir} \\
 & & 2145.9 & 3.32 & 36.81 & \bfseries -0.16 & \DNMF{dirA}\\
 & & 3685.3 & 3.87 & 54.42 & -0.00 & \DNMF{comb}\\
 & & 15396.0 & \bfseries 1.93  & 101.28  & 0.19 & \NMF{MU} \\
  & & 138.9 & 19.05 & 103.71 & 0.42 & \NMF{sp} \\
  & & 834.0 & 21.62 & 88.17 & 0.15 & \NMF{bf} \\[.7ex]  %
  AR & 20 & 15227.0 & 19.99 & 60.00  & 0.00 &\DNMF{ind} \\
   & & 13453.0 & 19.81 & 58.84 & 0.00 & \DNMF{indA}\\
 & & 5623.0 & 20.07 & \bfseries 16.71 &  \bfseries -0.28 & \DNMF{dir} \\
   & & 11118.8 & 20.54 & 20.15 & -0.28 & \DNMF{dirA}\\
    & & 13995.0 & 20.26 & 59.60 & -0.01 & \DNMF{comb}\\
  & & 22437.0 & \bfseries 11.16 & 98.10 & 0.48 & \NMF{MU}\\
 & & 50000.0 & 12.00  & 97.76  & 0.45 & \NMF{sp}\\
  & & 21.0 & 20.45  & 92.41  & 0.27 & \NMF{bf}\\[.7ex] %
   BRC & 40 & 9482.2 & 5.88  & 13.08 & 0.00  & \DNMF{ind} \\
    & & 9999.0 & 5.56 & 16.51 & 0.06 & \DNMF{indA}\\
  & & 9999.0 & 6.09  & 7.61  &  0.01  & \DNMF{dir} \\
  & & 9999.0 & 4.02 & 80.80 & 0.75 & \DNMF{dirA}\\
  & & 9543.8 & 6.53 & \bfseries 7.21 & \bfseries -0.01 & \DNMF{comb}\\
 & & 9999.0 & \bfseries 3.96  & 112.08  & 0.91  & \NMF{MU} \\
 & & 10000.0 & 4.08 & 108.64 & 0.93& \NMF{sp}\\
  & & 2.0 & 62.59 & 146.00 & 0.81 & \NMF{bf} \\[.7ex]  %
   NM & 40 & 9999.0 & 17.11 & 75.53  & 0.00 &\DNMF{ind}  \\
    & & 9999.0 & 15.21 & 74.83 & -0.01 & \DNMF{indA}\\
  & & 5486.0 & 11.03 & \bfseries 12.49 &  \bfseries -0.55 & \DNMF{dir} \\
    & & 3682.0 & 12.18 & 12.93 & -0.55 & \DNMF{dirA}\\
     & & 9999.0 & 18.05 & 74.36 & -0.02 & \DNMF{comb}\\
 & & 9999.0 & \bfseries 2.52 & 88.72 & 0.18 & \NMF{MU} \\
 & & 10000.0 & 3.75 & 73.96 & -0.07 & \NMF{sp} \\
  & & 860.0 & 42.21 & 78.65 & 0.01 & \NMF{bf} \\[.7ex]
  %
  SA & 20 & 10663.0 & 5.42 & 70.71  & 0.00 &\DNMF{ind} \\
   & & 17545.0 & 4.51 & 66.70 & -0.06 & \DNMF{indA}\\
  & & 18965.8 & 5.35 & \bfseries 18.12 &  \bfseries -0.43 & \DNMF{dir} \\
  & & 14102.7 & 4.74 & 25.63 & -0.35 & \DNMF{dirA}\\
  & & 9402.0 & 6.07 & 70.21 & -0.01 & \DNMF{comb}\\
  & & 46318.8 & \bfseries 1.68 & 119.02 & 0.36 & \NMF{MU}\\
  & & 50000.0 &  2.57 & 110.75 & 0.43 & \NMF{sp}\\
  & & 49.0 &  29.24 & 128.66 & 0.39 & \NMF{bf}\\[.7ex] %
  %
  4N & 60 & 1500.0 & 50.55 & 64.40 & 0.00 &\DNMF{ind} \\
   & & 1023.0 & 47.57 & 62.92 & 0.04 & \DNMF{indA}\\
  & & 1193.2 & 50.00 & 43.32 &  -0.12 & \DNMF{dir} \\
  & & 1221.6 & 49.88 & \bfseries 41.14 & \bfseries -0.15 & \DNMF{dirA}\\
  & & 3963.0 & 50.56 & 64.73 & 0.00 & \DNMF{comb}\\
 & & 1089.0 & \bfseries 37.22 & 114.15 & 0.53 & \NMF{MU}\\
 & & 34277.9 & 41.64 & 121.17  & 0.53 & \NMF{sp}\\
  & & 5.0 & 67.74 & 267.12  & 0.47 & \NMF{bf}\\
  \bottomrule
\end{tabular} \end{table}

Results presented in Tables \ref{tab:res},  \ref{tab:res1}, \ref{tab:res2} and \ref{tab:res3} are divided into four parts; NMF performance when supervised,  unsupervised rules, subgroups or redescriptions are used as constraints.  $\DNMF{ind}$ with normalized $\mF$ showed bad performance due to additional constraints imposed on $\mF$ and was left out from this evaluation. Functions proposed in equations  \ref{eq:dnmf} and  \ref{eq:dnmf3} have almost identical empirical performance, thus we provide in depth comparison only of functions from equations \ref{eq:dnmf1} and \ref{eq:dnmf}.  Empirical experiments determined that $\DNMF{combA}$ requires very high regularisation to achieve satisfactory descriptive properties, causing bad representation accuracy. Thus, we omit this approach from the tables in the results section.

\begin{table}[tb]
	\centering
	\caption{Comparative results of different NMF approaches, using $10$ different random initializations, when subgroups are used as constraints. Column names are as in Table \ref{tab:res}.}
	\label{tab:res2}
\begin{tabular}{@{}
  l%
  S[table-format=2.0]%
  S[table-format=5.1,group-minimum-digits=4]%
  S[table-format=2.2]%
  S[table-format=3.2]%
  S[table-format=+1.2]%
  L%
  @{}}
  \toprule
  {$\mathcal{D}$} & {$k$} & {\#iters} & {$\mathit{RE}$}  &  {$\mathit{DE}$} & {$\mathit{ADC}$} & \text{algorithm}  \\
  \midrule
  AR & 20 & 2833.0 & 24.65 & 65.39  & 0.00 &\DNMF{ind} \\
   & & 3606.0 & 24.61 & 65.54 & 0.00 & \DNMF{indA}\\
 & & 12685.0 & 25.46 & \bfseries 2.70 &  \bfseries -0.43 & \DNMF{dir} \\
 & & 49999.0 & 25.64 & 3.54 & -0.43 & \DNMF{dirA}\\
 & & 13479.3 & 24.51 & 65.54 & 0.00 & \DNMF{comb}\\
  & & 6610.0 & \bfseries 20.24 & 72.49 & 0.10 & \NMF{MU}\\
 & & 26119.6 & 21.80  & 21.51  & -0.38 & \NMF{sp}\\
 & & 89.0 & 23.93  & 70.57  & 0.06 & \NMF{bf}\\[.7ex] %
     BRC & 25 & 7658.6 & 6.03  & 29.94 & 0.00  & \DNMF{ind} \\
    & & 8377.2 & 5.64 & 31.38 & 0.02 & \DNMF{indA}\\
  & & 9999.0 & 5.08  & 47.71  &  0.17  & \DNMF{dir} \\
  & & 9999.0 & 4.74 & 86.65 & 0.53 & \DNMF{dirA}\\
  & & 2684.0 & 6.82 & \bfseries 21.29 & \bfseries -0.04 & \DNMF{comb}\\
 & & 9999.0 & \bfseries 4.69  & 111.12  & 0.79  & \NMF{MU} \\
 & & 10000.0 & 4.95 & 106.60 & 0.85& \NMF{sp}\\
  & & 4.0 & 6.75 & 133.97 & 0.77 & \NMF{bf} \\[.7ex]  %
 SA & 20 & 3094.0 & 9.47 & 65.67  & 0.00 &\DNMF{ind} \\
  & & 10427.0 & 6.50 & 69.33 & 0.05 & \DNMF{indA}\\
  & & 163.0 & 18.53 & \bfseries 26.44 &  \bfseries -0.37 & \DNMF{dir} \\
  & & 5418.0 & 6.64 & 33.44 & -0.32 & \DNMF{dirA}\\
   & & 3144.0 & 9.47 & 65.69 & 0.00 & \DNMF{comb}\\
  & & 11104.0 & \bfseries 4.17 & 76.83 & 0.14 & \NMF{MU}\\
  & & 3.2 &  21.00 & 30.58 & -0.34 & \NMF{sp}\\
  & & 199.0 &  37.08 & 66.45 & 0.00 & \NMF{bf}\\[.7ex] %
  %
  \bottomrule
\end{tabular} \end{table}

Unconstrained NMF is the most accurate methodology with respect to the data representation (see Tables \ref{tab:res}, \ref{tab:res1}, \ref{tab:res2} and \ref{tab:res3}), however, it has a large constraint representation error and low correspondence between data representation and a set of provided rules.  All proposed approaches ($\DNMF{}$) create latent factors that are represented much better by the available set of rules compared to the NMF approaches that do not use constraints related to factor descriptions. The effect is large with the average increase in correspondence of \numrange{0.2}{0.3} for the supervised rules, \numrange{0.3}{0.5} increase for descriptive rules, and redescriptions  between the $\DNMF{ind}$ and NMF approaches not optimizing descriptive constraints. Subgroups are specific in that they mostly have large overlaps and very large supports. The approach with sparsity constraints performs well in such conditions and outperforms the $\DNMF{ind}$ approach on two datasets (although the $\DNMF{dir}$ approach outperforms the $\NMF{sp}$ on all datasets). We discovered that the $\DNMF{GBD}$ is especially efficient on SA when subgroups are used, with ADC of $-0.42$ and RE of $0.11$. Gradient descent based approaches are hard to tune and prone to divergence. HALS approaches are generally faster than multiplicative updates--based approaches with comparable performance. $\DNMF{indA}$ and $\DNMF{dirA}$ approaches outperform $\DNMF{ind}$ on majority of datasets (or offer different representation/description accuracy trade-off).
\begin{table}[H]
	\centering
	\caption{Comparative results of different NMF approaches, using $10$ different random initializations and redescriptions as constraints. Column names are as in Table \ref{tab:res}.}
	\label{tab:res3}
\begin{tabular}{@{}
  l%
  S[table-format=2.0]%
  S[table-format=5.1,group-minimum-digits=4]%
  S[table-format=2.2]%
  S[table-format=3.2]%
  S[table-format=+1.2]%
  L%
  @{}}
  \toprule
  {$\mathcal{D}$} & {$k$} & {\#iters} & {$\mathit{RE}$}  &  {$\mathit{DE}$} & {$\mathit{ADC}$} & \text{algorithm}  \\
  \midrule
AB & 5 & 8647.0 & 3.68  & 46.17 & 0.00 & \DNMF{ind} \\
 & & 9733.0 & 3.87 & 51.18 & 0.06 & \DNMF{indA}\\
& & 1353.0 & 3.43 & 47.02 &  0.04 & \DNMF{dir} \\
& & 1746.5 & 3.63 & \bfseries 39.51 & 0.03 & \DNMF{dirA}\\
& & 12440.0 & 3.48 & 46.22 & \bfseries -0.00 & \DNMF{comb}\\
 & & 14553.0 & \bfseries 1.94  & 104.90  & 0.43 & \NMF{MU} \\
  & & 13941.0 & 9.83 & 98.55 & 0.58 & \NMF{sp} \\
  & & 834.0 & 21.62 & 99.58 & 0.42 & \NMF{bf} \\[.7ex]  %
  AR & 10 & 1516.2 & 20.31 & 65.53  & 0.00 &\DNMF{ind} \\
   & & 11726.6 & 19.56 & 66.99 & 0.03 & \DNMF{indA}\\
 & & 4092.0 & 19.73 & \bfseries 31.34 &  \bfseries -0.18 & \DNMF{dir} \\
 & & 18975.1 & 18.68 & 37.48 & -0.05 & \DNMF{dirA}\\
 & & 1245.4 & 20.90 & 64.32 & -0.01 & \DNMF{comb}\\
  & & 12666.0 & \bfseries 15.48 & 115.25 & 0.43 & \NMF{MU}\\
 & & 34390.8 & 19.00  & 106.59  & 0.51 & \NMF{sp}\\
  & & 44.0 & 22.04  & 117.48  & 0.32 & \NMF{bf}\\[.7ex] %
  BRC & 50 & 9999.0 & 5.80  & 9.43 & 0.00  & \DNMF{ind} \\
    & & 9999.0 & 5.79 & 5.61 & -0.01 & \DNMF{indA}\\
  & & 5719.0 & 6.24  & \bfseries 0.00  & \bfseries -0.01  & \DNMF{dir} \\
  & & 9999.0 & 3.69 & 55.47 & 0.69 & \DNMF{dirA}\\
  & & 9999.0 & 6.26 & 9.40 & 0.00 & \DNMF{comb}\\
 & & 9999.0 & 3.57  & 102.60  & 0.92  & \NMF{MU} \\
 & & 10000.0 & \bfseries 3.56 & 102.86 & 0.92& \NMF{sp}\\
  & & 1.0 & 63.86 & 111.94 & 0.74 & \NMF{bf} \\[.7ex]  %
   NM & 10 & 9999.0 & 11.59 & 56.52  & 0.00 &\DNMF{ind}  \\
    & & 9250.6 & 9.06 & 58.80 & 0.12 & \DNMF{indA}\\
  & & 3650.0 & 10.55 & \bfseries 23.20 &  \bfseries -0.28 & \DNMF{dir} \\
  & & 7480.4 & 9.09 & 32.35 & -0.05 & \DNMF{dirA}\\
   & & 9999.0 & 10.98 & 56.02 & -0.00 & \DNMF{comb}\\
 & & 9999.0 & \bfseries 5.08 & 102.10 & 0.36 & \NMF{MU} \\
 & & 10000 & 7.10 & 101.15 & 0.49 & \NMF{sp} \\
  & & 3445.0 & 16.24 & 106.71 & 0.38 & \NMF{bf} \\[.7ex]
  %
  SA & 10 & 12428.0 & 4.78 & 69.25  & 0.00 &\DNMF{ind} \\
   & & 19560.0 & 3.90 & 65.83 & 0.02 & \DNMF{indA}\\
  & & 16321.9 & 4.86 & \bfseries 29.79 &  \bfseries -0.37 & \DNMF{dir} \\
   & & 26522.4 & 4.32 & 29.84 & -0.26 & \DNMF{dirA}\\
   & & 15823.0 & 4.83 & 69.14 & -0.00 & \DNMF{comb}\\
  & & 28184.8 & \bfseries 3.01 & 106.58 & 0.31 & \NMF{MU}\\
 & & 50000.0 &  3.51 & 103.73 & 0.35 & \NMF{sp}\\
  & & 99.0 &  33.42 & 106.11 & 0.29 & \NMF{bf}\\[.7ex] %
  %
  4N & 60 & 1531.0 & 52.27 & 86.41 & 0.00 &\DNMF{ind} \\
   & & 2461.0 & 51.74 & 76.41 & -0.07 & \DNMF{indA}\\
  & & 1171.7 & 51.50 & 47.94 &  -0.35 & \DNMF{dir} \\
  & & 994.4 & 51.61 & \bfseries 46.50 & \bfseries -0.36 & \DNMF{dirA}\\
  & & 1969.0 & 53.85 & 83.96 & -0.04 & \DNMF{comb}\\
 & & 1445.0 & \bfseries 37.12 & 111.30 & 0.27 & \NMF{MU}\\
& & 37727.6 & 41.25 & 116.25  & 0.27 & \NMF{sp}\\
& & 5.0 & 67.74 & 237.56  & 0.20 & \NMF{bf}\\[.7ex]
WC & 20 & 31800.6 & 0.75 & 54.26 & 0.00 &\DNMF{ind} \\
 & & 30194.6 & 0.80 & 51.40 & 0.00 & \DNMF{indA}\\
  & & 39822.8 & 1.06 & 49.81 &  \bfseries -0.05 & \DNMF{dir} \\
  & & 49999.0 & 1.24 & \bfseries 43.48 & -0.03 & \DNMF{dirA}\\
  & & 13011.0 & 1.12 & 52.89 & -0.02 & \DNMF{comb}\\
 & & 5085.0 & \bfseries 0.61 & 118.37 & 0.66 & \NMF{MU}\\
& & 15307.2 & 31.23 & 163.32  & 0.45 & \NMF{sp}\\
& & 6.0 & 4.40 & 109.05  & 0.67 & \NMF{bf}\\[.7ex]
BO & 30 & 49999.0 & 5.68 & 71.34 & 0.00 &\DNMF{ind} \\
 & & 49999.0 & 5.54 & 65.94 & 0.22 & \DNMF{indA}\\
  & & 43196.0 & 6.10 & 14.15 &  \bfseries -0.32 & \DNMF{dir} \\
  & & 49999.0 & 5.83 & \bfseries 13.77 & -0.12 & \DNMF{dirA}\\
  & & 49999.0 & 5.89 & 72.07 & -0.10 & \DNMF{comb}\\
 & & 16694.0 & \bfseries 5.45 & 184.66 & 0.36 & \NMF{MU}\\
& & 338.6 & 28.58 & 112.21  & 0.39 & \NMF{sp}\\
& & 114.0 & 6.99 & 211.77  & 0.34 & \NMF{bf}\\
  \bottomrule
\end{tabular} \end{table}

The approach $\DNMF{dir}$ achieves additional \numrange{0.2}{0.3} gain compared to the $\DNMF{ind}$. This means that if the regular NMF creates latent factors that have around \SI{10}{\percent} entities in common with the associated set of rules, using our approach, loosing some accuracy (mostly \SI{<10}{\percent} of $\norm{\mX}_{F}$), the domain expert can gain factors that have \SIrange{60}{80}{\percent} entities in common with these rules. This allows performing scientific analyses and interpreting the obtained factors. We notice a serious effect of sparseness-related constraints on representation accuracy, however the  proposed approaches mostly manage to balance between these constraints and representation/descriptive accuracy to achieve better results in descriptive and comparable or better results in representation aspect. Methodology provides better correspondence when descriptive (unsupervised) rules and redescriptions are used. Descriptive rule sets are larger with higher rule diversity and larger deviation in support set sizes. They offer better part-based representation than predictive rules with large support sets or subgroups with very large support sets.

Factor description accuracy plots for a subset of datasets (Figure \ref{fig:factAcc}) display the correspondence of factors, obtained using different NMF approaches, with a predefined set of rules. We chose two datasets (Arrythmia and Sports articles) that allowed computing all four types of rules. It can be seen that $\DNMF{dir}$ contains a large number of highly accurate factors, however, it also produces small amount of factors whose accuracy varies greatly between different runs. $\DNMF{ind}$ mostly creates  factors with average or high correspondence with rule descriptions and has small variability between runs (descriptive accuracy can be increased on majority of datasets using higher regularization at the expense of representation accuracy). $\NMF{bf}$ has worse descriptive performance than the proposed approaches and similar or worse representation error. $\NMF{sp}$ performs well when subgroups are used as constraints. This probably occurs due to very large support sets of discovered subgroups, thus assigning all entities to all factors yields very good results.  

\begin{figure*}[tb!]
  \centering
  \includegraphics{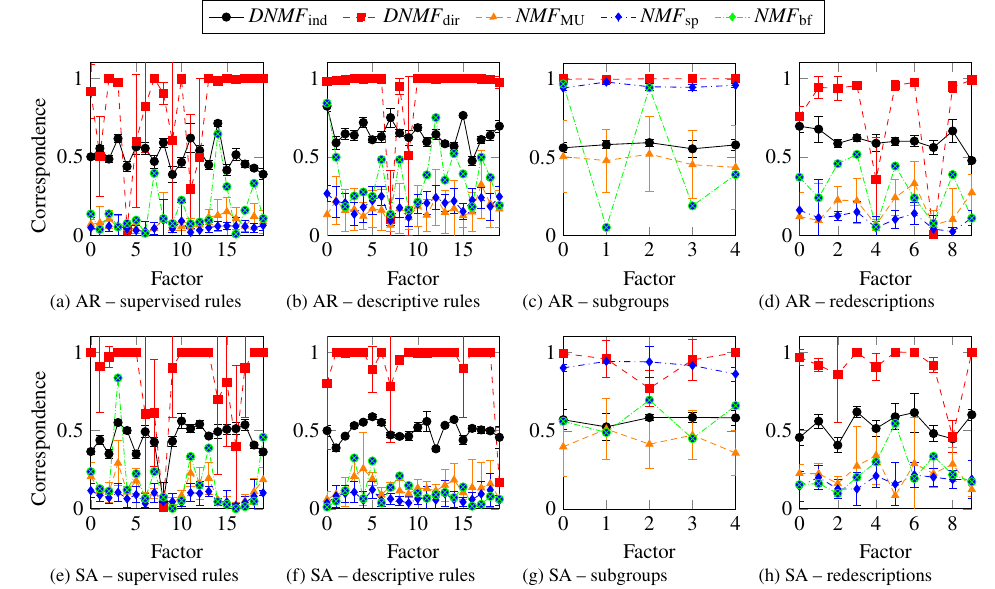}
  \caption{Factor correspondence accuracy achieved by the $\DNMF{ind}$, $\DNMF{dir}$, $\NMF{MU}$ and $\NMF{sp}$ approaches on three different datasets using supervised rules (above) and descriptive rules (below).}
\label{fig:factAcc}
\end{figure*}

\subsection{Example factor descriptions}
\label{sec:exp:example}

In this section, we show several examples of factor descriptions and explain the relation between latent factors and their associated rules. Several factor descriptions are presented in Table \ref{tab:desc}. Possible rule types (r.t.) are: supervised rules (S), descriptive rules (D), redescriptions (R) and subgroups (Sbg).

\begingroup
\setlength{\tabcolsep}{2.5pt}
\begin{table}[tb!]
\caption{Factor descriptions obtained on the Abalone (AB) and Sports Articles (SA) dataset one the fist run. Columns are factor id ($f$), rule type (r.t.), correspondence value ($\corr$) written in the format ($\corr_{\DNMF{ind}},\ \corr_{\DNMF{dur}}$), the factor description (Description) and the rule id ($r$).}
\label{tab:desc}
\begin{tabular}{@{}lrcclr@{}}
  \toprule
  $\mathcal{D}$ & $f$ & r. t. & $\corr$ &  Description & $r$ \\
  \midrule
AB & $4$ &S & $(0.73,\ 0.12)$ & \texttt{ShWeight} $\leq 0.2$  & $r_0$   \\[2mm]
  & $1$ & D  & $(0.68,\ 0.96)$ & \texttt{Diameter} $> 0.4\ $  & $r_1$ \\
 & &  &  & $\wedge\ $ \texttt{ShWeight} $\leq 0.4$ & \\
 & &  &  & $\wedge\ $ \texttt{ShcWeight} $> 0.2$ & \\[2mm]
   &  &  &  & \texttt{ShcWeight} $> 0.2\ $  & $r_2$ \\
 & &  &  & $\wedge\ $ \texttt{Diameter} $> 0.4$ & \\[.7ex]
  & $2$ & R  & $(0.61,\ 0.94)$ & $q_{0,0}:\ \neg ( 0.4\leq$ \texttt{ShWeight} $\leq 1$  & $R_0$ \\
  & &  &  & $\wedge\ 0.1\leq$ \texttt{Height} $\leq 0.2)$\\
   &  &  &  & $\vee\ ( 0\leq$ \texttt{ShWeight} $\leq 0.4\ $  &  \\
   &  &  &  & $\wedge\  0\leq$ \texttt{Height} $\leq 0.2)$  \\[1mm]
   &  &  &  & $q_{1,0}:\ 0.1\leq $ \texttt{Diameter} $\leq 0.5$  &  \\
 & &  &  & $\wedge\ 0\leq$ \texttt{ShcWeight} $\leq 0.5$ & \\
  & &  &  & $\wedge\ 0\leq$ \texttt{WhWeight} $\leq 1.1$  & \\
    &  &  &  & $\wedge\ \ 0\leq $\texttt{VisWeight} $\leq 0.2\ $  &  \\
    & &  &  & $ \wedge\ 0.1\leq$ \texttt{Length} $\leq 0.6$ & \\[.7ex]
   SA &  $13$ &   S & $(0.47,\ 1.0)$  &  \texttt{NNPS} $> 42$ & $r_3$ \\
& $0$ & D & $(0.49,\ 1.0)$  & \texttt{POS} $\leq 18\ \wedge$ \texttt{FW} $\leq 42$ & $r_4$ \\
& &  &   & $ \wedge$ \texttt{pronouns1st} $> 0.0$ & \\
& &  &   & $ \wedge$ \texttt{totalWordsCount} $\leq 400$ & \\[2mm]
& &  &   & \texttt{Quotes} $\leq 22$ & $r_5$\\
& &  &   & $\wedge$ \texttt{pronouns1st} $> 0$ & \\
& &  &   & $ \wedge$ \texttt{semanticsubjscore} $\leq 4$ & \\[.7ex] 
& $1$ & Sbg & $(0.40,\ 1.0)$  &  $0 \leq$\texttt{questionmarks} $\leq 1$ & $r_6$\\
& &  &   & $ \wedge$ \texttt{ellipsis} $= 0$ & \\
& &  &   & $ \wedge$ \texttt{TOs} $= 0\ \wedge\ $ \texttt{JJS} $= 0$ & \\[.7ex] 
& $0$ & R & $(0.48,\ 1.0)$  & $q_{0,1}:\ 29 \leq$ \texttt{past} $\leq 220$ & $R_1$\\[1mm]
& &  &   & $q_{1,1}:\ 0 \leq $ \texttt{colon} $\leq 4 \ \wedge\   $  & \\
& &  &   & $22 \leq$ \texttt{INs} $\leq 89\ \wedge\ $  & \\
& &  &   & $ 107 \leq$ \texttt{MD} $\leq 315$ $\wedge\  $  & \\
& &  &   & $29 \leq$ \texttt{VB} $\leq 104$ & \\
  \bottomrule
\end{tabular}
\end{table}
\endgroup

All rules that form factor descriptions, presented in the used examples, have a form of monotone conjunctions, redescriptions also contain negations and disjunctions. Presented factors have a correspondence at least $0.4$ with their set of rules, with the exception of $\DNMF{dir}$ on Abalone (this is the example of factor where $\DNMF{ind}$ outperforms the $\DNMF{dir}$).
All factors are described with one or two rules (redescriptions). Factor $f_1$ obtained on Abalone dataset with descriptive rule-based constraints, is described with two rules: $r_1$ and $r_2$, which makes its full description $r_1\ \vee\ r_2$. The relation between $f_1$ and these rules, for both methods, is depicted in Figure \ref{fig:factDesc}. $\DNMF{ind}$ produces factors with precision $0.96$ ($71$ entities are associated to $f_1$ that are not described by corresponding rules) and recall $0.70$ ($837$ entities are described by $r_1\vee r_2$ but are not associated with $f_1$). $\DNMF{dir}$ creates highly accurate factor with precision $0.996$ ($10$ wrongly assigned entities to $f_1$) and recall $0.96$ ($103$ entities not assigned to $f_1$).  %

\begin{figure}
\begin{subfigure}{.20\textwidth}
 \includegraphics[width=\linewidth]{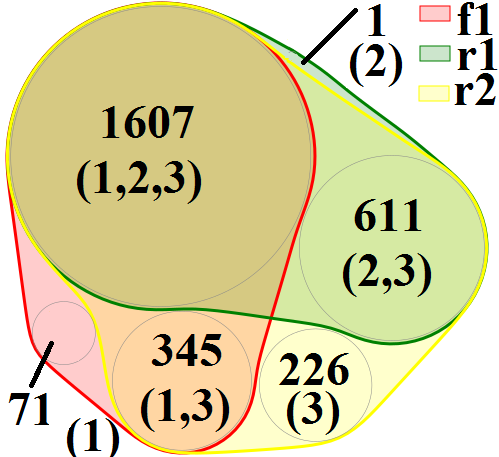}
\end{subfigure}\hspace{1cm}%
\begin{subfigure}{.2\textwidth}
 \includegraphics[width=\linewidth]{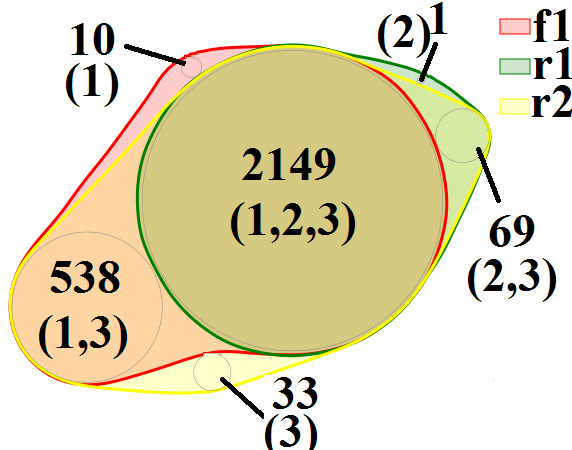}
\end{subfigure}
\caption{Relation between factor $f_1$ and the rules $r_1,\ r_2$ for the $\DNMF{ind}$ method (left) and the $\DNMF{dir}$ method (right).}
\label{fig:factDesc}
\end{figure}

Descriptions of obtained factors can be analysed further. For instance, supervised rules describing factor $f_{12}$, $r_{f_{12},1}\vee r_{f_{12},2}\vee r_{f_{12},3}$, where $r_{f_{12},1} = \text{heartRate} \leq 55 $, $r_{f_{12},2} = \text{heartRate} \leq 59\ \wedge\ \text{RWaveDII} \leq 8.9 $ and $r_{f_{12},3} = \text{heartRate}\leq 57\ \wedge\ \text{SwaveAVR}\leq 0\ \wedge\ \text{sex} = 1$, obtained on the Arrhythmia dataset describe the heart condition of $38$ individuals from which five  are diagnosed as normal and $33$ as not normal. It can be seen from the rules that the heart rate of the described subjects is below normal. This condition, called bradycardia can be found in young and healthy adults (especially athletic) but there also exist a serious condition called the sick sinus syndrome. Indeed, a majority $22$ patients contained in this group are diagnosed with sinus bradycardia, one  with left and four  with right bundle branch block, two  with ischemic changes and four  with other unspecified conditions.  

Descriptive rules do not use target label information and are not necessarily highly homogeneous with respect to the target label. This applies to $r_{A_{dsc}} = \text{RWaveV4} \leq 17.1\ \wedge\ \text{RWaveV3}\leq 12.9 \ \wedge\ \text{RWaveV2}\leq 6.7\ \wedge\ \text{RWaveV1} >24.0$, describing $70$ subjects with normal heart condition and $47$ with not normal condition. This descriptive rule is one of the rules describing factor $f_0$ on the Arrhythmia dataset.

\section{Related Work}
\label{sec:related-work}

Non-negative matrix factorization \cite{paatero94nmf,lee99nmf} was primarily designed to enable obtaining part-based representations in different tasks, especially where non-negativity was in a problem nature. Various tasks in image and signal processing, different problems in biology, bioinformatics, pharmacy, physics, medicine etc.\ contain non-negative data and have been analyzed using the NMF (see \cite{cichocki09nonnegative} and references therein). It was argued that non-negativity is important in human perception and as it turns out non-negative factors are often easier to understand and interpret than factors containing negative values \cite{pascual2006bionmf,Karthik08NMFCB}. NMF is now a mature research field with multiple developed optimization procedures, loss functions and different regularizers specifically suited for different tasks \cite{cichocki09nonnegative}. 

Among these regularizers, sparseness constraints \cite{Hoyer04Spars} and ortogonality constraints \cite{DingOrt06} significantly increase overall interpretability of the NMF-produced latent factors (see \cite{Karthik08NMFCB}). The main advantage of the sparseness constraints, with respect to interpretability, is that it allows reducing the number of non-zero elements used to represent the data. Orthogonality constraint allows obtaining clear clustering assignments equivalent to that obtainable by the $k$-means clustering algorithm. Although very useful, orthogonality constraints are often to strict and sparseness constraints can still be hard to understand because they contain non-negative real values in the constructed matrices. Slawski et al. \cite{Slawski13NMFBC} propose the NMF approach in which the basis elements are constrained to be binary. This allows detecting overlapping clusters and it provides latent factors with binary values, which are naturally easy to understand. The closest variant of NMF to the present work is the guided NMF \cite{GNMF}; see Section \ref{sec:algorithms} for more in-depth discussion.

The Boolean matrix factorization with background knowledge \cite{TRNECKA2022108261} utilizes user-provided weights to obtain factorization involving factors relevant to the domain expert. Orthogonal approach of verifying embeddings with hybrid logic \cite{shakya2023verification}, allows specifying and verifying different properties of deep learning embeddings using hybrid logic. Both approaches have very different goals to the proposed approach that obtains latent factors described by a set of rule-like objects. 

\section{Use case}
\label{sec:usecase}

Our use case example demonstrates the benefits of using the proposed methodology in a task of gene function prediction \cite{SurveyGFP}. This very important task in computational biology includes developing techniques, representations and using machine learning algorithms with the goal of obtaining more accurate predictions of gene functions. The ultimate goal is to understand the role, importance and impact of genes and gene groups on the functioning and structure of the organism.  

The task is being tackled using a multitude of different approaches that use data describing different aspects of organisms containing genes of interest \cite{SurveyGFP}. There is also a competition aimed at developing better approaches to solve this task \cite{Cafa}. We focus on one sub-area that aims to develop computational models for gene function prediction using information about genomes of different organisms \cite{SurveyGFP}. 

Our datasets contain $3475$ Clusters of Orthologous Groups (COGs) and Non-supervised Orthologous Groups (NOGs) \cite{COG}. These are groups of related genes that are known to share many functions. Gene functions are organized in a hierarchy known as Gene Ontology (GO) \cite{GO}. Each gene is assigned a number of functions (GO-terms), forming a hierarchical multi-label classification problem \cite{HMC}. Our datasets contain $1047$ different gene functions.

We use three different representations for gene function prediction using genomes of $1669$ prokaryotic organisms. The first representation called the \emph{phyletic profiles} \cite{VidulinComp} describes OGs by their membership in different bacterial organisms ($2731$ Boolean features). The \emph{gene neighbourhood} \cite{Neighbourhoods} approach uses information about physical distances of OGs in genomes of different bacterial organisms ($3475$ numerical features) and the \emph{neighbourhood function profiles} \cite{MihelcicBio} measure the average occurrence of gene functions in a neighbourhood of each OG throughout genomes of different bacteria ($1048$ numerical features). Thus, we have three different data sources that normally constitute three different approaches for predicting gene functions using information about genomes of different organisms. 

The proposed approach retains benefits such as dimensionality, sparsity reduction and non-negativity often found useful in prediction tasks where interpretation is also desirable. We will show that additionally, it can: a) achieve data fusion in an interpretable manner, b) produce representation that allows creating classifiers with better prediction of a number of gene functions compared to the original \emph{phyletic profiles} approach. This is achieved by utilizing information obtained from all three data sources. This is important since such, potentially complementary approaches can be used in synergy to achieve overall, globally better predictions \cite{VidulinComp, MihelcicBio} than achievable by individual components. 

We created $10$ non-negative  $\NMF{MU}$ and $\DNMF{dir}$ representations of the phyletic profiles, containing $30$ factors. Different random initialization is used in each run. Redescriptions obtained on gene neighbourhood and neighbourhood function profiles are used as constraints to create the $\DNMF{dir}$ representation of the phyletic profiles of the $3475$ OGs. We aim to extract knowledge that is supported by both sources, thus getting more robust knowledge. Using different types of rules allows obtaining different representations targeting different aspects of the problem.  

The phyletic profiles representation is sparse containing Boolean values, thus both representations have relatively high representation errors $RE_{\NMF{MU}} = 43.2\pm 0.05$ and $RE_{\DNMF{dir}} = 49.2\pm 0.5$, but the $\DNMF{dir}$ has significantly higher correspondence $0.62\pm 0.04$ compared to $0.04\pm 0.01$ obtained by the $\NMF{MU}$. We conclude that $\DNMF{dir}$ contains knowledge from all three data sources whereas the $\NMF{MU}$ contains only knowledge from phyletic profiles. 

Finally, we compare the number of functions per run for which we measure improvement in the AUPRC score \cite{AUPRC} of the Predictive Clustering tree algorithm \cite{HMCPCT} of at least $0.05$ using features obtained from $\NMF{MU}$ and $\DNMF{dir}$ compared to the performance this algorithm achieved using original phyletic profiles representation (See Figure \ref{fig:usecase}).  Since both NMF representations contain equivalent number of features, these are directly comparable. The difference in the number of improved functions is significant according to the one-sided Wilcoxon signed-rank test \cite{wilcoxon} ($p = 0.00256$).

\begin{figure}
  \centering
  \includegraphics{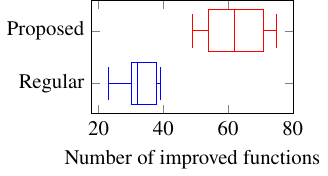}
  \caption{Number of improved functions across $10$ different runs compared to phyletic profile approach by $\NMF{MU}$ (bottom) and $\DNMF{dir}$ (top).}
\label{fig:usecase}
\end{figure}

The proposed approach also offers potential to use described data fusion in an interpretable manner. The factor $f_{13}$ of the obtained representation using $\DNMF{dir}$ is described with redescription $R_{f_{13}} = (q_{1,f_{13}}, q_{2,f_{13}})$, where $q_{1,f_{13}} = \  \neg (10.69\leq \text{NOG091082}\leq 25.0\ \wedge\ 5.65\leq\text{COG2965}\leq 25.0\ \wedge\ 11.49\leq \text{COG0050}\leq 14.5\ \wedge\ 9.72\leq \text{COG2882}\leq 25.0\ \wedge\ 11.98\leq \text{COG1344}\leq 14.29\  \wedge\ 5.3\leq \text{COG0806}\leq 14.31)$ and
$q_{2,f_{13}} = 0.0\leq \text{GO0046483}\leq 2.82\ \wedge\ 1.83 \leq \text{GO0043228}\leq 5.88$. $R_{f_{13}}$ describes $f_{13}$ with maximum correspondence $1.0$ which means that both queries are valid for all OGs that correspond to factor $f_{13}$. From these queries, one obtains knowledge about gene and functional composition of neighbourhoods of described OGs. But, there is a third information, that of procaryotic organisms that are associated with $f_{13}$. There are in total $18$ organisms such that their corresponding attributes contribute towards $f_{13}$ with intensity larger than $0.1$. Among these, there are $4$ strains of \emph{Buchnera} genus (Gram-negative bacteria that is a symbiont of aphids), $2$ strains of \emph{Bordetella} (Gram-negative bacteria that can infect humans), \emph{Wigglesworthia glossinidia} (Gram-negative bacteria that is endosymbiont of tsetse fly), a strain of \emph{Thermodesulfovibrio} (Gram-negative bacteria that is able to reduce  sulfate, thiosulfate or sulfite with a limited range of electron donors, found in hot springs), $5$ strains of \emph{Burkholderia} (Gram-negative, aerobic, rod-shaped, motile bacteria), \emph{Ammonifex degensii} (Gram-negative bacteria isolated from volcanic hot springs), \emph{Candidatus Desulforudis} (the only example of Gram-positive bacteria, sulfate-reducing, found in groundwater at high depths), \emph{Thermovibrio ammonificans} (Gram-negative,  thermophilic, anaerobic, chemolithoautotrophic bacterium found in deep-sea hydrothermal vent), \emph{Thermodesulfobacterium} (Gram-negative  thermophilic sulfate-reducing bacteria) and \emph{Desulfurobacterium} (Gram-negative, thermophilic, anaerobic, strictly autotrophic, sulphur-reducing bacterium). As we can see, the members are mostly Gram-negative bacteria or bacteria living in very high temperature equipped to process different chemical elements (mostly sulfates). This is reflected in the GO-functions that describe OGs associated to $f_{13}$, GO0043228 describes non-membrane-bounded organelle (not bounded by a lipid bilayer membrane) which is a definition of Gram-negative bacteria, whereas GO0046483 describes heterocycle metabolic process (the chemical reactions and pathways involving heterocyclic compounds, those with a cyclic molecular structure and at least two different atoms in the ring (or rings)) which corresponds to heavy chemical activity performed by many strains associated with $f_{13}$. From the query describing OG neighbourhood of OGs associated to $f_{13}$, the most understandable are COG2882 and COG1344 which are connected to cell motility (and there indeed are a number of strains like \emph{Burkholderia} that are motile). COG0050 and COG0806 are connected to translation, ribosomal structure and biogenesis and COG2965 with replication, recombination and repair. The rule describing OG neighbourhood is complicated (containing logical negation), thus all OGs that have the occurrence of any of the specified OGs outside the interval, marked inside the rule, are described by this rule.

\section{Conclusions}
\label{sec:conclusions}

This work presents a non-negative matrix factorization methodology that incorporates rule-based constraints to describe the resulting latent factors. The approach allows data fusion and provides advantages compared to available interpretable non-negative factorization approaches. We present and  evaluate three regularization terms and the corresponding multiplicative update rules that incorporate rule-based constraints into NMF optimization. The term that uses explicit information about the entity membership in support sets of different rules and the input entity-factor cost matrix uses the final entity-factor assignment indirectly ($\DNMF{ind}$). The $\DNMF{dir}$ utilizes information about the final entity-factor assignment directly and the $\DNMF{comb}$ combines properties of both aforementioned regularization terms. We provide theoretical and empirical study of properties of these approaches. Update rules are also presented in the greedy, oblique and HALS manner. Using proposed rule-clustering algorithm alleviates the need of experimentation to determine correct factor decomposition, with supervised rules, and integrates clustering and factorization. 

Experimental results performed on four different types of rules (supervised, descriptive conceptual, subgroups and redescriptions) show that the regularizer using final entity-factor assignment ($\DNMF{dir}$) mostly works better with predefined rule clusters associated to the latent factors. However, the approach using the cost matrix  ($\DNMF{ind}$)  usually has lower deviation in description accuracy between factors and shows promising results on high dimensional data. The $\DNMF{comb}$ has almost identical performance to the $\DNMF{ind}$. All three proposed approaches create latent factors with significantly higher correspondence with the input rule set than the competing approaches. Obtained results demonstrate that the proposed rule-clustering algorithm  offers competitive performance and often allows obtaining better result.  The overall representation accuracy achieved by the proposed approaches is expectedly lower compared to the regular NMF (varying between rule-sets and datasets). Although the proposed approach enforces a level of sparseness on matrix $\mF$, it substantially outperforms the approach with sparseness guarantees of this matrix. 

In the supervised tasks, the latent factors are often used as  features which are further used to increase accuracy of machine learning models. The presented use-case demonstrates that the proposed approach, utilizing information fusion, manages to create input features that improve predictive model performance on significantly larger number of target labels compared to features obtained from regular NMF approach. In this use-case, understanding the obtained features and their connection to the target label is of utmost importance for efficient study of the underlying problem and making further research decisions. We have demonstrated the in-depth knowledge provided by the proposed approach. In unsupervised tasks, the approach allows obtaining condensed representations, where latent factors form conceptual clusters.

The overall results of the approach could potentially be further improved by using rules produced by more recent approaches such as \cite{WangZLW24}.

\bibliographystyle{IEEEtran}
\begin{IEEEbiography}[{\includegraphics[width=1in,height=1.25in,clip,keepaspectratio]{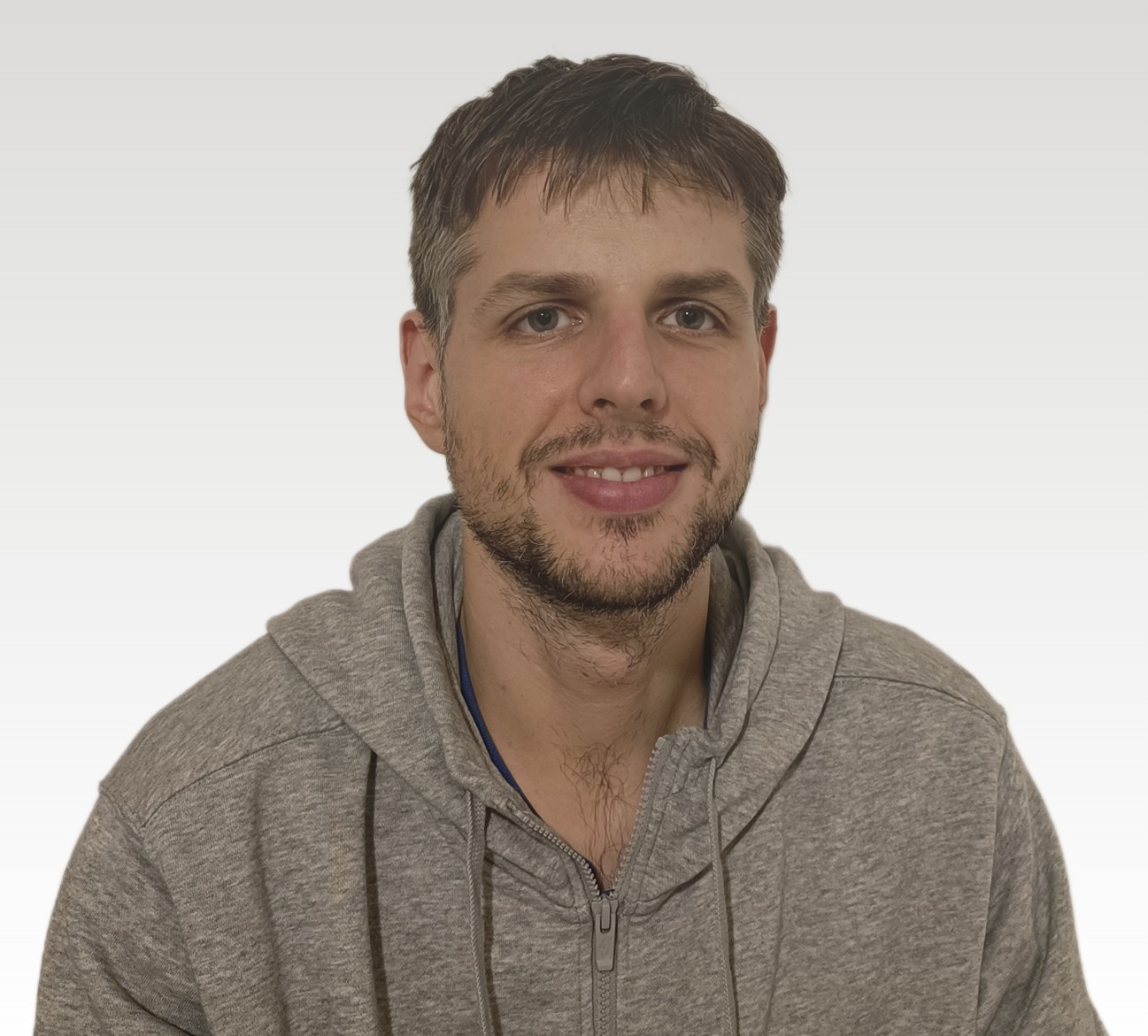}}]{Matej Mihel\v{c}i\'c}
 obtained his PhD degree in Computer Science in $2018$ from the International Postgraduate School Jožef Stefan in Ljubljana, Slovenia, while working at the Ruđer Bošković Institute in Zagreb, Croatia. From $2018-2021$ he was working as a postdoctoral researcher at the Department of Mathematics, Faculty of Science, University of Zagreb and from $2019 - 2020$ as a postdoctoral researcher at the School of Computing, University of Eastern Finland in Kuopio, Finland. He currently holds a position of the Assistant Professor at the department of Mathematics, Faculty of Science, University of Zagreb. He is an author of $21$ scientific articles. His research interests include redescription mining, interpretable machine learning, matrix factorization approaches and knowledge discovery.

Matej Mihelčić has served as a reviewer for $4$ different journals and $4$ different conferences and has been a PC member of ICDM, ECMLPKDD and PAKDD conferences.
\end{IEEEbiography}

\begin{IEEEbiography}[{\includegraphics[width=1in,height=1.25in,clip,keepaspectratio]{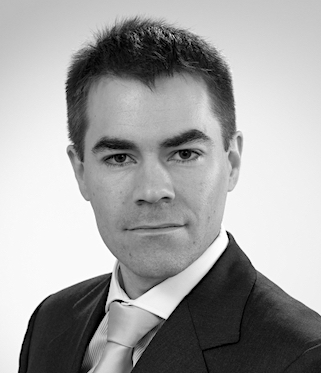}}]{Pauli Miettinen} received his PhD degree in computer science from University of Helsinki, Finland, in 2009. From 2010 to 2013 he was a Post-Doctoral Fellow and from 2013 until 2018 a Senior Researcher at Max-Planck-Institute for Informatics in Saarbr\"ucken, Germany. Since 2018, he has been a Professor of Data Science at University of Eastern Finland in Kuopio, Finland. He is an author of more than fifty peer reviewed articles. His research interests include matrix and tensor factorizations over non-standard algebras, graph mining and social network analysis, redescription mining, and applications of data mining techniques. He is an action editor of the journal \textit{Data Mining and Knowledge Discovery}.

  Prof. Miettinen has won three best paper awards. He was the program co-chair for 2021 IEEE International Conference on Data Mining. 
\end{IEEEbiography}

\includepdf[pages=-]{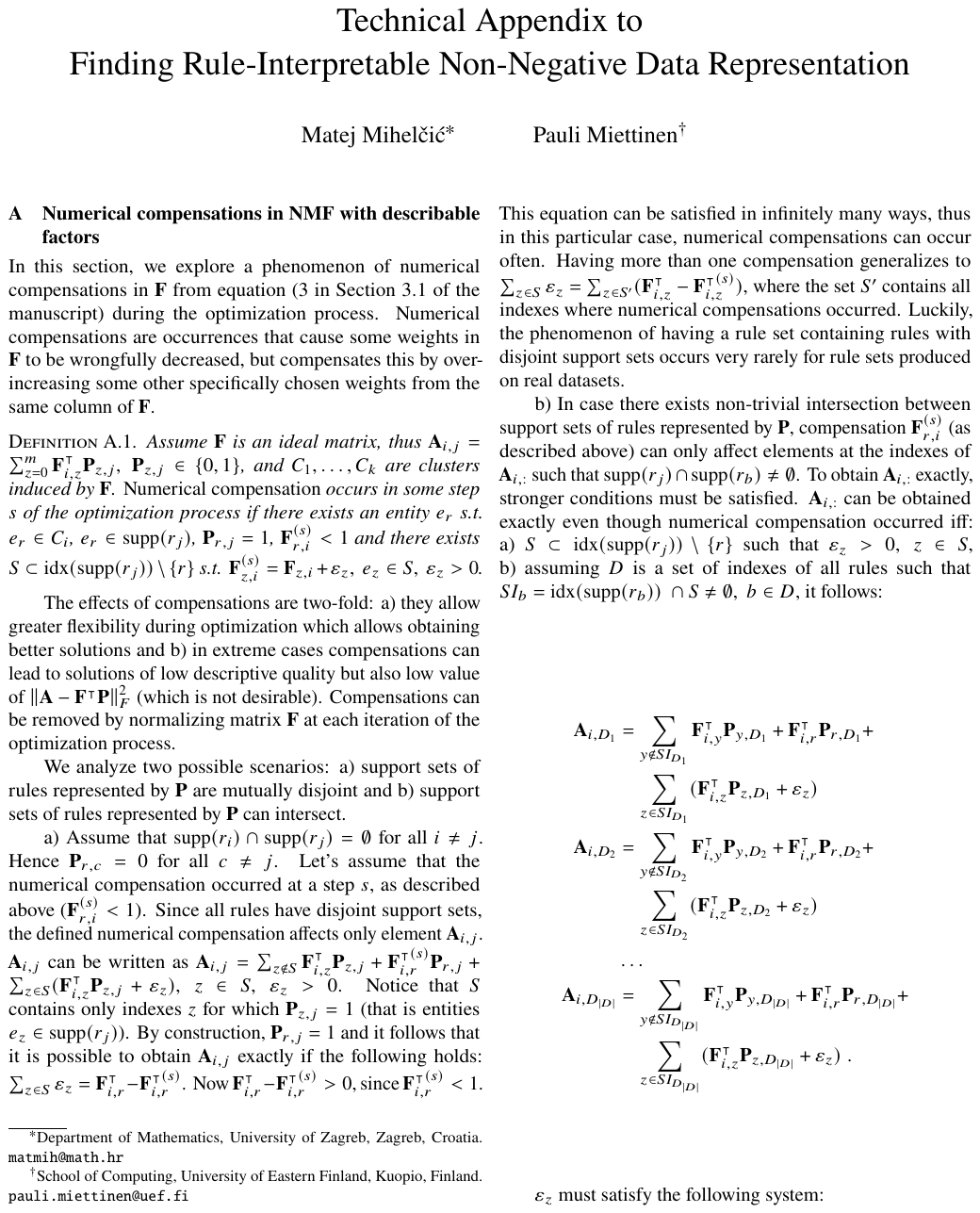}

\end{document}